\documentclass[10pt]{article}  

\usepackage{palatino}
\usepackage{subfigure}
\usepackage{enumerate}
\usepackage{bm}
\usepackage{amsmath}
\usepackage{amsthm}
\usepackage{amssymb}
\usepackage{wrapfig}
\usepackage{array}
\usepackage{color}
\usepackage{algorithmic}
\usepackage{graphics,url}
\usepackage{natbib}
\usepackage{subfigure} 
\usepackage{graphicx}
\usepackage{setspace}

\definecolor{orange}{rgb}{0.0,1,0}
\definecolor{red}{rgb}{1,0,0}

\newcommand{\FNormS}[1]{\mbox{}\left\|#1\right\|_{\mathrm{F}}^2}

\newcommand{\TNorm }[1]{\mbox{}\left\|#1\right\|_2  }

\usepackage{subfigure,enumerate,bm,amsmath,amsthm,wrapfig,array,color}
\usepackage{graphics}

\newtheorem{theorem}{\bf Theorem}[]
\newtheorem{lemma}[theorem]{Lemma}
\newtheorem{definition}[theorem]{Definition}

\newcommand{\transp}{\ensuremath{^\text{\textsc{t}}}}
\newcommand{\trace}{\text{\rm Tr}}
\newcommand{\mat}[1]{{\ensuremath{\textsc{#1}}}}

\def\rank{\hbox{\rm rank}}

\def\matA{\mat{A}}
\def\matB{\mat{B}}
\def\matC{\mat{C}}

\def\matG{\mat{G}}
\def\matH{\mat{H}}
\def\matI{\mat{I}}

\def\matQ{\mat{Q}}
\def\matR{\mat{R}}
\def\matS{\mat{S}}
\def\matU{\mat{U}}
\def\matV{\mat{V}}

\def\matX{\mat{X}}
\def\matY{\mat{Y}}
\def\matZ{\mat{Z}}

\DeclareMathSymbol{\Prob}{\mathbin}{AMSb}{"50}
\newcommand\remove[1]{}

\def\math#1{$#1$}

\def\mand#1{$$#1$$}
\def\frac#1#2{{#1\over #2}}

\def\mld#1{\begin{equation}
#1
\end{equation}}
\def\eqar#1{\begin{eqnarray}
#1
\end{eqnarray}}
\def\eqan#1{\begin{eqnarray*}
#1
\end{eqnarray*}}

\DeclareMathSymbol{\R}{\mathbin}{AMSb}{"52}

\def\textchoose#1#2{(\begin{smallmatrix}#1\\#2\end{smallmatrix})}

\def\aa{{\mathbf a}}
\def\bb{{\mathbf b}}
\def\cc{{\mathbf c}}
\def\ee{{\mathbf e}}
\def\ff{{\mathbf f}}
\def\gg{{\mathbf g}}
\def\hh{{\mathbf h}}
\def\uu{{\mathbf u}}
\def\vv{{\mathbf v}}

\def\xx{{\mathbf x}}

\def\zz{{\mathbf z}}

\def\aa{{\mathbf a}}
\def\bb{{\mathbf b}}

\def\norm#1{{\|#1\|}}

\def\ceil#1{{\left\lceil\,#1\,\right\rceil}}
\def\r#1{{(\ref{#1})}}

\def\dotfil{\leaders\hbox to 1.5mm{.}\hfill}
\newcounter{rmnum}
\def\RN#1{\setcounter{rmnum}{#1}\uppercase\expandafter{\romannumeral\value{rmnum}}}
\def\rn#1{\setcounter{rmnum}{#1}\expandafter{\romannumeral\value{rmnum}}}

\providecommand\remove[1]{}
\DeclareMathSymbol{\Prob}{\mathbin}{AMSb}{"50}
\DeclareMathSymbol{\Exp}{\mathbin}{AMSb}{"45}

\newcommand{\christos}[1]{\textcolor{orange}{CHRISTOS: #1}}

\usepackage[normalem]{ulem} 
\usepackage{soul}

\title{Optimal Sparse Linear Auto-Encoders and Sparse PCA
}
\author{Malik Magdon-Ismail\\RPI CS Department, Troy, NY\\{\sf magdon@cs.rpi.edu}\and Christos Boutsidis\\Yahoo Labs, New York, NY\\{\sf
boutsidis@yahoo-inc.com}
}

\begin{document}
\begin{spacing}{1.25}

\maketitle

\begin{abstract}%
\noindent
Principal components 
analysis~(PCA) is the optimal linear auto-encoder of data, and it 
is often used to construct features.
Enforcing sparsity on the principal components can promote better 
generalization, while improving the interpretability of the features.
We study the problem of constructing optimal sparse
linear auto-encoders.
Two natural questions in such a setting are:
\begin{enumerate}[(i)]\itemsep0pt
\item Given a level of sparsity, what is the best approximation to
PCA that can be achieved? 
\item Are there low-order 
polynomial-time algorithms which can asymptotically achieve this optimal 
tradeoff between the sparsity and the approximation quality?
\end{enumerate}
In this work, we answer both questions by
giving efficient low-order polynomial-time algorithms for constructing 
asymptotically \emph{optimal} linear auto-encoders (in particular, sparse
features with near-PCA reconstruction error) and
demonstrate the performance of our algorithms on real data.
\end{abstract}

\section{Introduction}
\label{section:intro}

An auto-encoder transforms (encodes)
the data into a low dimensional
space (the feature space) and then lifts (decodes) it back to 
the original space. The auto-encoder reconstructs the data through a 
bottleneck, and if the reconstruction is close to the original 
data, then the encoder was able to preserve most of the information
using just a small number of features. 
Auto-encoders are important in machine learning because they 
perform information preserving dimension reduction.
The decoder only plays a
minor role in verifying that the encoder didn't lose much information.
It is the encoder that is important and constructs the 
(useful) low-dimensional feature vector. 
Nonlinear auto-encoders played an important role in 
auto-associative 
neural networks~\citep{CM88,BH88,BK88,O91}. 
A special case is the linear auto-encoder in which 
both the decoder and encoder are linear maps~\citep{O92}. 
Perhaps the most famous linear auto-encoder is 
principal components analysis (PCA), in particular because it is 
optimal: PCA is the linear auto-encoder that preserves the maximum
amount of information (given the dimensionality of the feature space).
We study general linear auto-encoders, and enforce sparsity on the
encoding linear map on the grounds that a sparse encoder is easier to 
interpret. A special case of a sparse linear encoder is sparse PCA.

More formally, 
the data matrix is \math{\matX\in\R^{n\times d}} (each row
\math{\xx_i\transp \in \R^{1 \times d}}
is a data point in \math{d} dimensions).
Our focus is the \emph{linear} auto-encoder, which,
for \math{k< d}, is a pair of linear mappings 
\eqan{
&&h:\R^d\mapsto\R^k\qquad \text{and}\\
&&g:\R^k\mapsto\R^d,
}
specified by an encoder 
matrix \math{\matH\in\R^{d\times k}} and a decoder matrix
\math{\matG\in\R^{k\times d}}.
For data point \math{\xx\in\R^d}, the encoded feature
is 
\mand{\zz=h(\xx)=\matH\transp\xx \in \R^k} 
and the reconstructed datum is
\mand{\hat\xx=g(\zz)=\matG\transp\zz \in \R^d.} 
Using \math{\hat\matX\in\R^{n\times d}} to denote the 
reconstructed data matrix, we have 
\mand{\hat\matX=\matX\matH\matG.}
The pair \math{(\matH,\matG)} are a good auto-encoder if 
\math{\hat\matX\approx\matX}, using the squared loss: 
\begin{definition}[Information Loss \math{\ell(\matH,\matX)}]
The information loss of linear encoder \math{\matH} is the minimum 
possible reconstruction error for \math{\matX} over all linear decoders 
\math{\matG}:
\mand{
\ell(\matH, \matX)=
\min_{\matG\in\R^{k\times d}}\FNormS{\matX-\matX\matH\matG}
=
\FNormS{\matX-\matX\matH(\matX\matH)^\dagger\matX}.
}
The last formula follows by doing a linear regression to obtain the optimal
\math{\matG}
\end{definition}
PCA is perhaps the most famous 
linear auto-encoder, because it
is optimal with respect to information loss. 
Since \math{\rank(\matX\matH\matG)\le k}, the information loss is bounded
by 
\mand{\ell(\matH,\matX)\ge\FNormS{\matX-\matX_k},} 
(for a matrix \math{\matA},
\math{\matA_k} is its best rank-\math{k} approximation). 
By the 
Eckart-Young theorem \math{\matX_k=\matX\matV_k\matV_k\transp}, where
\math{\matV_k \in \R^{d \times k}} 
is the matrix whose columns are the top-\math{k}
right singular vectors of \math{\matX}
(see, for example, e-Chapter 9 of \citet{malik199}). 
Thus, the \emph{optimal}
linear encoder is \math{\matH_{opt}=\matV_k}, with 
optimal decoder \math{\matG_{opt}=\matV_k\transp};
and, the corresponding top-\math{k} PCA-features are
\math{\matZ_{\text{pca}}=\matX\matV_k}. 

Since its early beginings in 
\citet{P1901}, PCA has become a  classic tool
for data analysis, visualization and feature extraction.
While PCA simplifies the data by concentrating
as much information as possible into a few components, those components
may not be easy to interpret. In many applications,
it is desirable to ``explain'' the
PCA-features using a small number of the original 
dimensions because 
the original variables have direct physical 
significance. For example, in biological applications, they may be 
genes, or in financial applications they may be assets. 
One seeks a tradeoff between the \emph{fidelity} of the
features (their ability to reconstruct the
data), and the \emph{interpretability} of the
features using a few original variables.
We would like the encoder \math{\matH} to be \emph{sparse}. 
Towards this end, we introduce a sparsity parameter
\math{r} 
and require that every column of 
\math{\matH} have at most \math{r} non-zero elements.
\emph{Every} feature in an \math{r}-sparse 
encoding can be ``explained'' using at most
\math{r} original features.
Such interpretable factors are known
as 
\emph{sparse principal components (SPCA)}.
We may now formally state the 
\emph{sparse linear encoder problem} that we consider in this work: 
\begin{center}
\fbox{
\begin{minipage}{0.85\textwidth}
\underline{{\bf Problem: Optimal \math{r}-sparse encoder (Sparse PCA)}}
\\[0.1in]
Given \math{\matX\in\R^{n\times d}}, \math{\varepsilon>0}
and $k < \rank(\matA)$, 
find, for the smallest possible \math{r},
an \math{r}-sparse encoder \math{\matH} for which 
\mand{
\ell(\matH,\matX)=\norm{\matX-\matX\matH(\matX\matH)^\dagger\matX}_{\mathrm{F}}^2
\le(1+\varepsilon)\norm{\matX-\matX_k}_{\mathrm{F}}^2.}
\end{minipage}
}
\end{center}
Note that we seek a \emph{relative-error approximation} to the
optimal loss.

\subsection{Notation}
 
Let \math{\rho\le\min\{n,d\}=\rank(\matX)} 
(typically \math{\rho=d}).
We use \math{\matA,\matB,\matC,\ldots} for
matrices and \math{\aa,\bb,\cc,\ldots} for vectors. The standard Euclidean
basis vectors are \math{\ee_1,\ee_2,\ldots} (the dimension will usually be 
clear from the context).

The
singular value decomposition (SVD)
allows us to write
\math{\matX=\matU\Sigma\matV\transp}, where
the columns of \math{\matU\in\R^{n\times \rho}} are the
\math{\rho} left singular vectors, the columns of 
\math{\matV\in\R^{d\times \rho}} are the
\math{\rho} right singular vectors, and \math{\Sigma\in\R^{\rho\times\rho}} 
is a diagonal matrix of positive singular values \math{\sigma_1\ge\cdots\ge
\sigma_\rho}; \math{\matU} and \math{\matV} are orthonormal,
so \math{\matU\transp\matU=\matV\transp\matV=\matI_\rho} \cite{GV96}. 
For integer \math{k}, we use 
\math{\matU_k\in\R^{n\times k}} (resp. \math{\matV_k\in\R^{d\times k}}) 
for the first \math{k} left (resp. right)
singular vectors, and \math{\Sigma_k\in\R^{k\times k}} is the 
diagonal matrix of corresponding top-\math{k} singular values.
We can view a matrix as a row of columns. So, 
\math{\matX=[\ff_1,\ldots,\ff_d]}, 
\math{\matU=[\uu_1,\ldots,\uu_\rho]},
\math{\matV=[\vv_1,\ldots,\vv_\rho]},
\math{\matU_k=[\uu_1,\ldots,\uu_k]} and 
\math{\matV_k=[\vv_1,\ldots,\vv_k]}. We use \math{\ff} for the columns of
\math{\matX}, the \emph{features}, and we reserve~\math{\xx_i} 
for the data points (rows of \math{\matX}), 
\math{\matX\transp=[\xx_1,\ldots,\xx_n]}. We say that matrix
\math{\matA=[\aa_1,\ldots,\aa_k]} is (\math{r_1,\ldots,r_k})-sparse if
\math{\norm{\aa_i}_0\le r_i}; moreover, if all 
\math{r_i} are equal to \math{r},
we say the matrix is \math{r}-sparse.

The Frobenius (Euclidean) norm of a matrix \math{\matA} is 
\math{\FNormS{\matA}=\sum_{ij}\matA_{ij}^2=\trace(\matA\transp\matA)=
\trace(\matA\matA\transp)}. The pseudo-inverse 
\math{\matA^\dagger} of \math{\matA} with 
SVD \math{\matU_\matA\Sigma_\matA\matV_\matA\transp} is  
\math{\matA^\dagger=\matV_\matA\Sigma_\matA^{-1}\matU_\matA\transp};
\math{\matA\matA^\dagger=\matU_\matA\matU_\matA\transp} is a symmetric 
projection operator. For matrices \math{\matA,\matB} with 
\math{\matA\transp\matB=\bm0}, a generalized Pythagoras theorem holds,
\math{\norm{\matA+\matB}_{\mathrm{F}}^2=\norm{\matA}_{\mathrm{F}}^2+\norm{\matB}_{\mathrm{F}}^2}.
\math{\TNorm{\matA}} is
the operator norm (top singular value) of \math{\matA}.

\subsection{Outline of Results}

\paragraph{Polynomial-time Algorithm for Near-Optimal \math{r}-Sparse Encoder.}
We give the first polynomial-time algorithms to construct 
an \math{r}-sparse linear encoder
\math{\matH} and a guarantee that the
information loss is close to the
optimal information loss of PCA
(see Theorem~\ref{theorem:SPCA}),
\mand{
\ell(\matH,\matX)\le(1+O(k/r))\norm{\matX-\matX_k}_\mathrm{F}^2.
}
Setting \math{r=O(k/\epsilon)} gives a \math{(1+\epsilon)}-approximation to
PCA.
Our algorithm is efficient,
running in \math{O(ndr+(n+d)r^2)} time (low-order polynomial).
We know of no other result that provides a guarantee on the quality of a 
top-\math{k} sparse linear encoder with respect to the 
optimal linear encoder (PCA). 
Our algorithm constructs all \math{k}
factors simultaneously, using a blackbox reduction to column subset selection
(see Theorem~\ref{theorem:blackbox}).

\paragraph{Lower Bound on Sparsity to Achieve Near-Optimal Performance.}
We give the first lower bound on the sparsity \math{r} that is required
to achieve a \math{(1+\epsilon)}-approximation to the information
loss of PCA (see Theorem~\ref{theorem:lower}). The lower bound shows that
sparsity \math{r=\Omega(k/\epsilon)} is required to guarantee
\math{(1+\epsilon)}-approximate information loss with respect to PCA, and 
hence that our algorithm is asymptotically optimal.

\paragraph{Iterative Algorithm for Linear Encoder.}
Our algorithm constructs all \math{k}
factors simultaneously, in a sense treating all the factors equally. One
cannot identify a ``top'' factor. Most existing 
sparse PCA algorithms first construct a top sparse PCA factor; then, every
next sparse factor must be orthogonal to all previous ones.
We develop a similar iterative algorithm by running our ``batch'' algorithm
repeatedly with \math{k=1}, each time extracting a provably
accurate sparse factor for a residual. We give a performance 
guarantee for the resulting \math{k} factors that are constructed by our
algorithm (see Theorem~\ref{theorem:mainIterative}). We show that
in each step, all factors constructed up to that point are provably 
accurate.
Ours is the first performance guarantee for any iterative scheme of this
type, that constructs sparse factors one by one.

Our information loss guarantee for the iterative algorithm is approximately
a \math{(1+O(\varepsilon\log k))}-factor worse than PCA (up to a small
additive error).
This bound is not as good as for the batch algorithm, but, in practice,
the iterative
algorithm is faster, constructs sparser factors and performs well.

\paragraph{Experiments.}
We show the experimental performance of our algorithms on 
standard benchmark data sets and compared with some standard 
benchmark sparse PCA algorithms. The experimental
 results indicate that
our algorithms perform as the theory predicts, and in all
cases produces factors which are comparable or better to the
standard benchmark algorithms.

\subsection{Discussion of Related Work}
PCA is the most popular linear auto-encoder, due to its optimality.
Nonlinear auto-encoders became prominent with 
auto-associative
neural networks~\citep{CM88,BH88,BK88,O91,O92}. We are unaware of
work addressing the more gereral ``sparse linear auto-encoder''. However, 
there is a lot of research on ``sparse PCA'', a special case of a sparse
linear auto-encoder.

\paragraph{Sparse Factors.}
The importance of sparse factors in dimensionality reduction 
was recognized in some early work: 
the \emph{varimax} criterion of \cite{K58} was used to 
rotate the factors and encourage sparsity, and this
has been used in multi-dimensional scaling approaches to dimension
reduction by \cite{S69,K64}.
One of the first attempts at sparse PCA used 
axis rotations and thresholding~\citep{CJ95}. Since then,
sophisticated
computational techniques have been developed. 
In general, these methods address finding
just one sparse principal component, and one can apply the
algorithm iteratively on the residual after projection to get
additional sparse principal components.

\paragraph{Minimizing Information Loss versus Maximizing Variance.}
The traditional formulation of sparse PCA is as a cardinality
constrained variance maximization problem:
maximize \math{\vv\transp\matA\vv} subject to
\math{\vv\transp\vv=1} and \math{\norm{\vv}_0\le r}, for $\matA \in \R^{n \times n}$
and $r < n$.
A straightforward reduction from {\sc max-clique} shows this problem to
be NP-hard (if \math{\matA} is the adjacency matrix of a graph, then
\math{\matA} has a clique of size \math{k} if and only if
\math{\vv\transp\matA\vv\ge(k-1)} for a \math{k}-sparse 
unit vector \math{\vv}, see~\citet{M2015}).
Sparse PCA is a special case of a generalized eigenvalue problem:
maximize \math{\vv\transp\matS\vv} subject to
\math{\vv\transp\matQ\vv=1} and \math{\norm{\vv}_0\le r}.
This generalized eigenvalue problem is known to be 
NP-hard \cite{MGWA08}, via a reduction from 
sparse regression which is
NP-hard \cite{N95,FKT14}.

This view of PCA 
as the projection which maximizes variance is due to a
historical restriction to symmetric auto-encoders
\math{\matH\matH^\dagger}  (so \math{\matG=\matH^\dagger}), see for example
\cite{O92}.
The PCA auto-encoder is symmetric
because \math{\matV_k^\dagger=\matV_k\transp}. Observe that
\eqan{
var(\matX)=\norm{\matX}_{\mathrm{F}}^2&=&
\norm{\matX(\matI-\matH\matH^\dagger)+\matX\matH\matH^\dagger}_{\mathrm{F}}^2\\
&=&
\norm{\matX-\matX\matH\matH^\dagger}_{\mathrm{F}}^2+\norm{\matX\matH\matH^\dagger}_{\mathrm{F}}^2,
}
where the last equality is from Pythagoras' theorem.
Minimizing the information loss  
\math{\norm{\matX-\matX\matH\matH^\dagger}_{\mathrm{F}}^2}
is \emph{equivalent} to maximizing 
\math{\norm{\matX\matH\matH^\dagger}_{\mathrm{F}}^2
=\trace(\matH^\dagger\matX\transp\matX\matH)}, 
the \emph{symmetric explained variance}. 
(A similar decomposition holds for the information loss of 
a general linear auto-encoder and 
the ``true'' explained variance).
The top-\math{k} principal components \math{\matV_k}
maximize the symmetric explained variance.
This view of PCA as the projection that captures the maximum 
symmetric explained variance has led to the historical approach to sparse PCA:
find a symmetric autoencoder that is sparse and captures the maximum
symmetric explained variance. 
The decomposition of the variance into a sum of information loss and
symmetric explained variance means that 
in an unconstrained setting,
minimizing information loss and maximizing the symmetric explained variance  
are both ways of 
encouraging \math{\matH} to be close to 
\math{\matV_k}. 
However, when \math{\matH} is constrained (for example
to be sparse), these optimization objectives
can produce \emph{very} different optimal solutions, and 
there are several reasons
to focus on minimizing information loss:
\begin{enumerate}[(i)]
\item
Maximizing variance corresponds to minimizing information loss for
symmetric decoders. For the general encoder however, variance has no intrinsic
value, but information loss directly captures the unsupervised
machine learning goal:
the decoder is secondary and what matters is that the
encoder produce a compact representation of the data and preserve
\emph{as much information as possible}. Constraints on the decoder 
(such requiring a symmetric auto-encoder) will translate
into a suboptimal encoder that loses more information than is neccessary.
We are after an encoder into a lower dimensional space that preserves as much
information in the data as possible. Hence,
to get the most informative sparse features, one should
directly minimize the information loss (placing no constraints on the
decoder). 

\item
An approximation algorithm for information loss can be converted
to an approximation algorithm for variance maximization.
\begin{theorem}
\label{theorem:var-approx}
If \math{\FNormS{\matX-\matX\matH\matH^\dagger}\le
(1+\varepsilon)\FNormS{\matX-\matX_k}}, then 
\mand{\FNormS{\matX\matH\matH^\dagger}\ge\FNormS{\matX_k}-
\varepsilon\FNormS{\matX-\matX_k}
\ge {\textstyle\left(1-\frac{\rho-k}{k}\varepsilon\right)}\norm{\matX_k}_\mathrm{F}^2.
}
\end{theorem}
\begin{proof}
Since \math{\FNormS{\matX-\matX\matH\matH^\dagger}=
\FNormS{\matX}-\FNormS{\matX\matH\matH^\dagger}\le(1+\varepsilon)
\FNormS{\matX-\matX_k}}, we have
\mand{
\FNormS{\matX\matH\matH^\dagger}
\ge\FNormS{\matX}-(1+\varepsilon)
\FNormS{\matX-\matX_k}=\FNormS{\matX_k}-\varepsilon\FNormS{\matX-\matX_k}.
}
The second inequality follows from the bound  
\math{\FNormS{\matX-\matX_k}\le\FNormS{\matX_k}\cdot(\rho-k)/k}.
\end{proof}
Thus, a relative error approximation for reconstruction
error gives a relative error approximation for explained variance.
However, a decoder which gives a relative
error approximation for symmetric explained variance
does not immediately give a relative error approximation for information 
loss.

\item 
Explained variance 
is not well defined for general encoders, whereas information loss
is well defined.
As \citet{ZHT06} points out,
one has to be careful when defining the explained variance for
general encoders.
The interpretation of 
\math{\trace(\matH^\dagger\matX\transp\matX\matH)}
as an explained variance 
relies on two important properties:
the columns in \math{\matH} 
are orthonormal (``independent'') directions and they are 
decorrelated from each other,
that is
\math{\matH\transp\matH=\matI_k} and \math{\matH\transp\matX\transp
\matX\matH} is diagonal. The right singular vectors
are the unique factors having these two properties.
Therefore, when one introduces a cardinality constraint, one has to give up
one or both of these properties, and typically one relaxes the 
decorrelation requirement. Now, as  
\citet{ZHT06} points out, when factors are correlated, the variance is
not straightforward to define, and 
\math{\trace(\matH^\dagger\matX\transp\matX\matH)} is an optipistic
estimate of explainied variance. 
\citet{ZHT06} computes 
a ``variance after decorrelation'' to quantify the quality of
sparse PCA. 
Their solution is not completely satisfactory since the
order in which factors are sequentially
decorrelated can change the explained variance.
Our solution is simple: use the information loss which is not only
the natural metric one is interested in, but is always well defined.
\end{enumerate}

\paragraph{Heuristics for Sparse PCA via Maximizing Variance.}
Despite the complications with with interpreting the explained variance
with general encoders which produce correlated factors, all the
heuristics we are aware of has focussed on maximizing the variance.
With a sparsity constraint, the problem 
becomes combinatorial and 
the exhaustive algorithm 
requires \math{O\bigl(dr^2 \textchoose{d}{r}\bigr)} computation.
This exponential running time can be
improved to \math{O(d^{q+1})} for a
rank-\math{q} perturbation of the identity~\citep{APK11}. 
None of these exhaustive 
algorithms are practical for high-dimensional data.
Several heuristics exist.
\citet{TJU03} and \citet{ZHT06} take an \math{L_1} penalization view. 
DSPCA (direct sparse PCA)  
\cite{AEJL07} also uses an \math{L_1} sparsifier but
solves a relaxed convex
semidefinite program which is further refined in
\citet{ABE08} where they also
tractable sufficient
condition for testing optimality.
The simplest algorithms  
use greedy forward
and backward subset selection. For example, \cite{MWA06a} develop
a greedy branch and bound algorithm based on spectral bounds with 
\math{O(d^3)} running time for forward selection 
and \math{O(d^4)} running time for backward selection.
An alternative view of the problem is as a sparse matrix reconstruction
problem;  for example \cite{SH08} obtain sparse principal components
using regularized low-rank matrix approximation.

\paragraph{Theoretical guarantees.}
The heuristics
are quite mature, but few theoretical 
guarantees are known. 
We are not aware of any polynomial time algorithms with guarantees 
on optimality. In \cite{APD2014}, the sparse PCA problem is considered
with an additional non-negativity constraint; and the authors 
give an algorithm which takes as input a parameter
\math{k}, has running time \math{O(d^{k+1}\log d+d^kr^3)} 
and constructs a sparse solution that is a
\math{(1-\frac{n}{r}\norm{\matS-\matS_k}_2/\norm{\matS}_2)}-factor from
optimal. The running time is not practical when \math{k} is large
and the approximation guarantee is only non-trivial when the 
spectrum of \math{\matS} is rapidly decaying.
Further, the approximation guarantee only applies to constructing 
the first sparse PCA and it is not clear how this result can be extended
if the algorithm is applied iteratively.

To our knowledge,
there are no known guarantees for top-\math{k} sparse factors
with respect to PCA 
Within the \emph{more} general setting of
sparse linear auto-encoders,
we solve two open problems:
(\rn{1}) We determine the best approximation
\emph{guarantee} with respect to the optimal linear
encoder (PCA) that can be achieved using a linear encoder with
sparsity \math{r}; (\rn{2}) We give low order polynomial algorithms that
achieve an approximation guarantee with respect to PCA using a sparsity
that is within a factor of two of the minimum required sparsity. 
Our algorithms apply to constructing
\math{k} sparse
linear
features
with performance comparable to top-\math{k} PCA.

\section{Optimal Sparse Linear Encoder}

We show a black-box reduction of sparse linear encoding to 
column subset selection, and so
it is worthwhile to first discuss background in column subset
selection. 
We then use column subset selection 
algorithms to construct provably accurate sparse auto-encoders. 
Finally, we 
modify our main algorithm so that it can select features iteratively and prove
a bound for this iterative setting.

\subsection{Column Subset Selection Problem (CSSP)}

For \math{\matX=[\ff_1,\ldots,\ff_d]},  we let 
\math{\matC=[\ff_{i_1},\ff_{i_2}\ldots,\ff_{i_r}]} denote a matrix formed using 
\math{r}
columns ``sampled'' from \math{\matX}, where  
\math{1\le i_1< i_2\cdots< i_r\le d} are distinct column indices. 
Column sampling is a linear 
operation, and we can use a matrix 
\math{\Omega \in \R^{d \times r}} to perform the column sampling,
\mand{\matC=\matX\Omega,
\qquad\text{where}\qquad
\Omega=[\ee_{i_1},\ee_{i_2}\ldots,\ee_{i_r}],
}
and \math{\ee_i} are the standard basis vectors in \math{\R^{d}} 
(post-multiplying \math{\matX}
by \math{\ee_i} ``samples'' the  \math{i}th column of \math{\matX}).
The columns of \math{\matC} span a subspace in the range of 
\math{\matX} (which is the span of the columns of \math{\matX}).
Any column sampling matrix can be used to construct an
\math{r}-sparse matrix.
\begin{lemma}\label{lem:r-sparse}
Let \math{\Omega=[\ee_{i_1},\ee_{i_2}\ldots,\ee_{i_r}]  \in \R^{d \times r} } and let 
\math{\matA\in\R^{r\times k}} be any matrix. Then 
\math{\Omega\matA \in \R^{d \times k}} is \math{r}-sparse, i.e.,
each column of $\Omega\matA$ has at most $r$ non-zero entries. 
\end{lemma}
\begin{proof}
Let \math{\matA=[\aa_1,\ldots,\aa_k]} and 
consider \math{\Omega\aa_j} which is column 
\math{j} of \math{\Omega\matA}. A zero-row
of \math{\Omega} results in the corresponding entry in \math{\Omega\aa_j} 
being zero. Since \math{\Omega} has \math{r} non-zero rows, there are
at most \math{r} non-zero entries in \math{\Omega\aa_j}.
\end{proof}
Given columns \math{\matC}, we define \math{\matX_{\matC}=\matC\matC^\dagger\matX} 
to be the matrix obtained by projecting all the columns of \math{\matX} onto
the subspace spanned by \math{\matC}. For any matrix \math{\hat\matX}
whose columns 
are in the span of \math{\matC}, 
\math{\norm{\matX-\matX_{\matC}}_{\mathrm{F}}^2\le \norm{\matX-\hat\matX}_{\mathrm{F}}^2}.
Let \math{\matX_{\matC,k} \in \R^{n \times d}} be the optimal rank-\math{k} approximation to
\math{\matX_\matC} obtained via the SVD of \math{\matX_\matC}. 
\begin{lemma}[See, for example, \cite{malik195}]\label{lem:XCk}
\math{\matX_{\matC,k}} is a rank-\math{k} matrix whose columns are in the
span of \math{\matC}. 
Let \math{\hat\matX} be any rank-\math{k} matrix whose columns are in the 
span of \math{\matC}. Then,  
\math{\norm{\matX-\matX_{\matC,k}}_{\mathrm{F}}^2\le \norm{\matX-\hat\matX}_{\mathrm{F}}^2}.
\end{lemma}
That is, \math{\matX_{\matC,k}} is the best rank-\math{k} approximation to 
\math{\matX} whose columns are in the span of \math{\matC}.
An efficient algorithm to compute \math{\matX_{\matC,k}} is also given in 
\cite{malik195}. The algorithm runs in \math{O(ndr+(n+d)r^2)} time. 
We reproduce that
algorithm here.
\begin{center}
\fbox{
\begin{minipage}{0.8\textwidth}
\underline{{\bf Algorithm to compute \math{\matX_{\matC,k}}}}
\\[0.1in]
{\bf Inputs:} \math{\matX\in\R^{n\times d}}, \math{\matC\in\R^{n\times r}},
\math{k\le r}.

{\bf Output:} \math{\matX_{\matC,k}\in\R^{n\times d}}.
\begin{algorithmic}[1]
\STATE Compute a \math{\matQ\matR}-factorization of \math{\matC} as
\math{\matC=\matQ\matR}, with $\matQ \in \R^{n \times r},$ 
$\matR \in \R^{r \times r}$. 
\STATE Compute \math{(\matQ\transp\matX)_k \in \R^{r \times d}} via SVD (the best rank-\math{k}
approximation to \math{\matQ\transp\matX}).
\STATE Return \math{\matX_{\matC,k}=\matQ(\matQ\transp\matX)_k\in\R^{n\times d}}.
\end{algorithmic}
\end{minipage}
}
\end{center}
We mention that it is possible to compute a 
\math{(1+\epsilon)}-approximation to \math{\matX_{\matC,k}} more quickly 
using randomized projections~\cite[section 3.5.2]{BW14full}.

\subsection{Sparse Linear Encoders from CSSP}\label{sec:encoders}

The main result of this section is to show that if we can obtain a set of 
columns \math{\matC} for which \math{\matX_{\matC,k}} is a good approximation
to \math{\matX}, then we can get a good sparse linear encoder for $\matX$. We first give
the algorithm for obtaining an \math{r}-sparse linear encoder
from \math{r}-columns \math{\matC}, 
and then we give the approximation guarantee. For simplicity, in the 
algorithm below we assume that \math{\matC} has full column rank. 
This is not essential, and if \math{\matC} is rank deficient, then 
the algorithm can be modified to first remove the dependent
columns in \math{\matC}.
\begin{center}
\fbox{
\begin{minipage}{0.8\textwidth}
\underline{{\bf Blackbox algorithm to compute encoder from CSSP}}
\\[0.1in]
{\bf Inputs:} 
\math{\matX\in\R^{n\times d}}; 
\math{\matC\in\R^{n\times r}} with \math{\matC=\matX\Omega} and 
\math{\Omega=[\ee_{i_1},\ldots,\ee_{i_r}]};
\math{k\le r}.

{\bf Output:} \math{r}-sparse linear encoder 
\math{\matH\in\R^{d\times k}}.

\begin{algorithmic}[1]
\STATE Compute a \math{\matQ\matR}-factorization of \math{\matC} as
\math{\matC=\matQ\matR}, with $\matQ \in \R^{n \times r},$ 
$\matR \in \R^{r \times r}$.
\STATE Compute the SVD of \math{\matR^{-1}(\matQ\transp\matX)_k},
\math{\matR^{-1}(\matQ\transp\matX)_k=\matU_{\matR}\Sigma_{\matR}\matV_{\matR}\transp},
where 

\centerline{
\math{\matU_{\matR}\in\R^{r\times k}},\qquad
\math{\Sigma_{\matR}\in\R^{k\times k}}\qquad and \qquad
\math{\matV_{\matR}\in\R^{d\times k}}.}
\vspace*{3pt}
\STATE Return \math{\matH=\Omega\matU_{\matR}\in\R^{d\times k}}.
\end{algorithmic}
\end{minipage}
}
\end{center}
In the algorithm above, since \math{\matC} has full column rank, \math{\matR}
is invertible. Note that the algorithm can be modified to 
accomodate \math{r<k}, in which case it will output fewer than
\math{k} factors in the output encoding.
In step 2, even though \math{\matR^{-1}(\matQ\transp\matX)_k} is an
\math{r\times d} matrix, it has rank \math{k}, hence the dimensions of 
\math{\matU_{\matR},\ \Sigma_\matR,\ \matV_\matR} depend on \math{k}, not \math{r}.
By Lemma~\ref{lem:r-sparse}, the
encoder \math{\matH} produced by the algorithm is \math{r}-sparse. 
Also observe that \math{\matH} has orthonormal columns, as is typically
desired for an encoder:
\mand{
\matH\transp\matH=\matU_{\matR}\transp\Omega\transp\Omega\matU_{\matR}
=\matU_{\matR}\transp\matU_{\matR}=\matI_k.
}
Actually,
the encoder \math{\matH} has a much stronger property than 
\math{r}-sparsity, namely that in every
column the non-zeros can only be located at the \emph{same} \math{r}
coordinates. We now compute the running time of the algorithm. The first two
steps are as in the algorithm to compute \math{\matX_{\matC,k}} and
so take \math{O(ndr+(n+d)r^2)} time
(the multiplication by \math{\matR^{-1}} and the computation of an additional
SVD do not affect the asymptotic running time). 
The last step involves two matrix
multiplications which can be done in \math{O(r^2k+drk)} additional time,
which also 
does not affect the asymptotic running time of \math{O(ndr+(n+d)r^2)}.
We now show that the encoder \math{\matH} produced by the algorithm is
good if the columns result in a good rank-\math{k}
approximation \math{\matX_{\matC,k}}.
\begin{theorem}[Blackbox encoder from CSSP]\label{theorem:blackbox}
Given \math{\matX\in\R^{n\times d}}, \math{\matC=\matX\Omega \in \R^{n \times r}} with
\math{\Omega=[\ee_{i_1},\ldots,\ee_{i_r}]} and \math{k\le r}, let
\math{\matH}  be the \math{r}-sparse 
linear encoder produced by the 
algorithm above, which runs in  \math{O(ndr+(n+d)r^2)} time. Then,
the information loss satisfies
\mand{
\ell(\matH,\matX)=\norm{\matX-\matX\matH(\matX\matH)^\dagger\matX}_{\mathrm{F}}^2
\le
\norm{\matX-\matX_{\matC,k}}_{\mathrm{F}}^2.
}
\end{theorem}
The theorem says that if we can find a set of \math{r} columns within
which a good rank-\math{k} approximation to \math{\matX} exists, then
we can construct a good linear sparse encoder. 

\begin{proof}
We show that there is a decoder \math{\matG} for which
\math{\norm{\matX-\matX\matH\matG}_{\mathrm{F}}^2=\norm{\matX-\matX_{\matC,k}}_{\mathrm{F}}^2.} 
The theorem then follows because the optimal decoder cannot have higher 
reconstruction error. We use \math{\matU_\matR}, \math{\Sigma_\matR} and
\math{\matV_\matR} as defined in the algorithm, and set
\math{\matG=\Sigma_\matR\matV_\matR\transp}.
We then have,
\eqan{
\matX\matH\matG
&=&\matX\Omega\matU_\matR\Sigma_\matR\matV_\matR\transp\\
&=&\matX\Omega\matR^{-1}(\matQ\transp\matX)_k\\
&=&\matC\matR^{-1}(\matQ\transp\matX)_k
}
The first step is by construction in the algorithm because 
\math{\matU_\matR,\ \Sigma_\matR,\ \matV_\matR} is obtained from the SVD of 
\math{\matR^{-1}(\matQ\transp\matX)_k}.
The second step is because \math{\matC=\matX\Omega}.
Since \math{\matC=\matQ\matR} (and \math{\matR} is invertible),
it follows that \math{\matC\matR^{-1}=\matQ}. Hence,
\eqan{
\matX\matH\matG
&=&\matQ(\matQ\transp\matX)_k\\
&=&\matX_{\matC,k},
}
hence 
\math{\norm{\matX-\matX\matH\matG}_{\mathrm{F}}^2=\norm{\matX-\matX_{\matC,k}}_{\mathrm{F}}^2},
concluding the proof.
\end{proof}

What remains is to find a sampling matrix \math{\Omega} which gives a
good set of columns \math{\matC=\matX\Omega} for which
\math{\norm{\matX-\matX_{\matC,k}}_{\mathrm{F}}^2} is small. The main tool to obtain
\math{\matC} and \math{\Omega} 
was developed in \cite{malik195} which gave a constant factor
deterministic approximation algorithm and a relative-error randomized 
approximation algorithm. We state a simplified form of
the  result and then discuss
various ways in which this result can be enhanced. The main point is that
any algorithm to construct a good set of columns can be used 
as a black box to get a sparse linear encoder.
\begin{theorem}[Near-optimal CSSP \cite{malik195}]\label{theorem:Cmain}
Given \math{\matX\in\R^{n\times d}} of rank \math{\rho} and target rank 
\math{k}: 
\begin{enumerate}[(i)]
\item (Theorem 2 in \cite{malik195}) For sparsity parameter \math{r>k}, 
there is a deterministic algorithm which runs in time
\math{T_{V_k}+O(ndk+dk^3)} to construct a sampling matrix
\math{\Omega=[\ee_{i_1},\ldots,\ee_{i_r}]} and corresponding 
columns \math{\matC=\matX\Omega} such that
\mand{
\norm{\matX-\matX_{\matC,k}}_{\mathrm{F}}^2\le\left(1+\frac{1}{(1-\sqrt{k/r})^2}\right)
\norm{\matX-\matX_{k}}_{\mathrm{F}}^2.
}
\item (Simplified Theorem 5 in \cite{malik195})
For sparsity parameter \math{r>5k}, 
there is a randomized algorithm which runs in time
\math{O(ndk+dk^3)} to construct a sampling matrix
\math{\Omega=[\ee_{i_1},\ldots,\ee_{i_r}]} and corresponding 
columns \math{\matC=\matX\Omega} such that
\mand{
\Exp\left[\norm{\matX-\matX_{\matC,k}}_{\mathrm{F}}^2\right]\le\left(1+\frac{5k}{r-5k}\right)\norm{\matX-\matX_{k}}_{\mathrm{F}}^2.
}
\end{enumerate}
\end{theorem}
\begin{proof}(Sketch for part (\rn{2}).)
We follow the proof in \cite{malik195}, using the 
same notation in \cite{malik195}. Choose \math{5k} columns
for the initial constant factor approximation and set the 
accuracy \math{\varepsilon_0} for the approximate SVD to compute
 \math{\hat\matV_k} to
a constant so that
\math{c_0=(1+\varepsilon_0)(1+(1-\sqrt{1/5})^{-2})=5}. The adaptive sampling step 
ensures a \math{(1+c_0k/s)}-approximation for the reconstruction,
where \math{s=r-5k}.
\end{proof}
\paragraph{Comments.}
\begin{enumerate}[1.]
\item In part (\rn{2}) of the theorem, 
setting \math{r=5k+5k/\varepsilon=O(k/\varepsilon)} ensures that 
\math{\matX_{\matC,k}} is a  
\math{(1+\varepsilon)}-reconstruction of \math{\matX_k} (in expectation).
At the expense of an increase in the running
time the sparsity can be reduced to \math{r\approx 2k/\varepsilon} while
still giving a \math{(1+\varepsilon)}-reconstruction.
The message is that \math{O(k/\varepsilon)}-sparsity suffices to get  
a \math{(1+\varepsilon)}-accurate reconstruction.
\remove{
Thus,
one can find an approximation to the top-PCA vector having at most about
20 non-zero loadings and a reconstruction error that is at most
10\% of the PCA reconstruction error. This holds for 
\emph{any} matrix \math{\matX}.
}

\item Part (\rn{2}) of the theorem gives a bound on 
\math{\Exp[\norm{\matX-\matX_{\matC,k}}_{\mathrm{F}}^2]}. The expectation is with respect
to random choices in the algorithm. Using an application of
Markov's inequality to the positive random variable
\math{\norm{\matX-\matX_{\matC,k}}_{\mathrm{F}}^2
-\norm{\matX-\matX_{k}}_{\mathrm{F}}^2}, 
 \math{\norm{\matX-\matX_{\matC,k}}_{\mathrm{F}}^2\le(1+2\varepsilon)
\norm{\matX-\matX_{k}}_{\mathrm{F}}^2} holds 
with probability at least \math{\frac12}. 
This can be boosted to high-probability,
at least \math{1-\delta}, for an additional \math{\log\frac1\delta} factor
increase in the
running time.

\item
In part (\rn{1}) of the theorem
\math{\matV_k} (the top-\math{k} right singular vectors) are 
used in the algorithm, and  \math{T_{V_k}} is the time to compute 
\math{\matV_k}.
To compute \math{\matV_k} exactly, the only known algorithm is via the 
full SVD of \math{\matX}, which takes \math{O(nd\min\{n,d\})} time. 
It turns out that an approximate 
SVD will do. Suppose that \math{\hat\matV_k} satisfies
\mand{
\norm{\matX-\matX\hat\matV_k\hat\matV_k\transp}_{\mathrm{F}}^2\le\alpha
\norm{\matX-\matX_{k}}_{\mathrm{F}}^2.
}
The approximation \math{\hat\matV_k} can be used in the algorithm instead of
\math{\matV_k} and the error will increase by an additional factor of 
\math{\alpha}. \citet{malik195} gives an efficient
randomized approximation for \math{\matV_k} and there is also a
 recent deterministic algorithm 
given in \cite[Theorem 3.1]{GP13} which
computes such an approximate \math{\hat\matV_k} in time
\math{O(ndk\varepsilon^{-2})} with \math{\alpha=1+\varepsilon}.

\item 
The details of the algorithm that achieves part (\rn{2}) of the theorem
are given in~\cite{malik195}.
We simply state the three main steps.
\begin{enumerate}[(i)]
\item Use a 
randomized projection to compute an approximation \math{\hat\matV_k}.
\item Apply the algorithm from  part (\rn{1}) of the theorem using
\math{\hat\matV_k} to
get a constant factor approximation.
\item Apply one round of randomized adaptive column 
sampling~\cite{DV06} which boosts the constant factor approximation to a 
\math{1+\varepsilon} approximation.
\end{enumerate}
This entire algorithm can be de-randomized to give a deterministic algorithm
by using the deterministic approximation for \math{\hat\matV_k} from 
\cite{GP13} and a derandomization of the adaptive sampling step which 
appeared in a recent result~\cite{BW14}. The tradeoff will be a non-trivial
increase in the running time. 
The entire algorithm can also be made to run in input-sparsity-time
(see \cite{BW14full} for details).

\item {\bf Doubly sparse auto-encoders.} 
It is possible to 
construct an \math{O(k/\varepsilon)}-sparse 
linear auto-encoder and a corresponding
``sparse'' decoder with the following property:
the reconstruction \math{\matX\matH\matG} has rows which are
all linear combinations of \math{O(k/\varepsilon)} rows of \math{\matX}. 
Thus, the encoding identifies features which have loadings on only a few
dimensions, and the decoder for the encoding is based on a small number of
data points. Further, the reconstruction error is near optimal,
\math{\norm{\matX-\matX\matH\matG}_{\mathrm{F}}^2
\le(1+\varepsilon)\norm{\matX-\matX_{k}}_{\mathrm{F}}^2}.

We give the rough sketch for this doubly sparse linear auto-encoder, without 
getting into too many details. The main tool is the \math{\matC\Phi\matR} 
matrix
decomposition~\cite{DK03}. An asymptotically optimal
construction was given in \cite{BW14} which constructs columns
\math{\matC=\matX\Omega_1\in\R^{n\times r_1}} 
and rows \math{\matR=\Omega_2\transp\matX\in\R^{r_2\times d}}
and a rank-\math{k} matrix \math{\Phi\in\R^{r_1\times r_2}} for which 
\math{\norm{\matX-\matC\Phi\matR}_{\mathrm{F}}^2\le 
(1+\varepsilon)\norm{\matX-\matX_k}_{\mathrm{F}}^2},
where \math{\Omega_1=[\ee_{i_1},\ldots,\ee_{i_{r_1}}]} and 
\math{\Omega_2=[\ee_{i_1},\ldots,\ee_{i_{r_2}}]} are sampling matrices
with \math{r_1=O(k/\varepsilon)} and \math{r_2=O(k/\varepsilon)}.
Let \math{\Phi=\matU_\Phi\Sigma_\Phi\matV_\Phi\transp} be the 
SVD of \math{\Phi}, where 
\math{\matU_\Phi\in\R^{r_1\times k}},
\math{\Sigma_\Phi\in\R^{k\times k}},
\math{\matV_\Phi\in\R^{r_2\times k}}.
Then, we set the encoder to 
\math{\matH=\Omega_1\matU_\Phi} which is \math{O(k/\varepsilon)}-sparse,
and the decoder to
\math{\matG=\Sigma_\Phi\matV_\Phi\transp\Omega_2\transp\matX}.
The reconstruction is 
\math{
\matX\Omega_1\matU_\Phi\Sigma_\Phi\matV_\Phi\transp\Omega_2\transp\matX
};
The matrix \math{\Omega_2\transp\matX} contains 
\math{O(k/\varepsilon)} rows of \math{\matX} and so the reconstruction
is based on linear combinations of these rows.

The trade off for getting this doubly-sparse encoding is a decrease in the
accuracy, since the decoder is now only getting an approximation to the
best rank-\math{k} reconstruction within the column-span of \math{\matC}.
The advantage of the doubly sparse encoder is that it
identifies a small set of \math{O(k/\varepsilon)} 
original features that are important for 
the dimension-reduced features. It also simultaneously
identifies a small number of \math{O(k/\varepsilon)} 
data points that are important for reconstructing 
the entire data matrix.
\end{enumerate}
We are now ready for the main theoretical result of the section. By
using Theorem~\ref{theorem:Cmain} in our black-box linear encoder, we obtain our
algorithm to construct a sparse linear encoder with a provable 
approximation guarantee.
\begin{center}
\fbox{
\begin{minipage}{0.85\textwidth}
\underline{{\bf Sparse Linear Encoder Algorithm}}
\\[0.1in]
{\bf Inputs:} 
\math{\matX\in\R^{n\times d}}; target rank
\math{k\le\rank(\matX)}; sparsity \math{r>k}.

{\bf Output:} Near-optimal \math{r}-sparse linear encoder 
\math{\matH\in\R^{d\times k}}.

\begin{algorithmic}[1]
\STATE Use the algorithm from Theorem~\ref{theorem:Cmain}-(\rn{2}) 
to compute
columns \math{\matC=\matX\Omega\in\R^{n\times r}}, with inputs 
\math{\matX,\ k,\ r}.
\STATE Return the encoder
\math{\matH} computed by using 
\math{\matX,\ \matC,\ k} as input to the CSSP-blackbox encoder algorithm.
\end{algorithmic}
\end{minipage}
}
\end{center}
Using Theorem~\ref{theorem:Cmain} in
Theorem~\ref{theorem:blackbox}, we have an approximation guarantee
for our algorithm.
\begin{theorem}[Sparse Linear Encoder]\label{theorem:SPCA}
Given \math{\matX\in\R^{n\times d}} of rank \math{\rho}, the target 
number of sparse PCA vectors \math{k\le\rho},
and sparsity parameter \math{r>5k}, 
there is a randomized algorithm running in time
\math{O(ndr+(n+d)r^2+dk^3)} 
which constructs an \math{r}-sparse encoder \math{\matH} such that:
\mand{
\Exp[\ell(\matH,\matX)=\Exp\left[\norm{\matX-\matX\matH(\matX\matH)^\dagger\matX}_{\mathrm{F}}^2\right]
\le
\left(1+\frac{5k}{r-5k}\right)
\norm{\matX-\matX_{k}}_{\mathrm{F}}^2.
}
\end{theorem}
\paragraph{Comments.}
\begin{enumerate}[1.]
\item The expectation is over the random choices made in the
algorithm to construct \math{\matH}.
All the comments from Theorem~\ref{theorem:Cmain} apply here as well.
In particular, this result is easily converted to a high probability
guarantee, or even a deterministic guarantee at the expense of some accuracy
and a (polynomial in \math{d,r})-factor increase in the
 computational cost.

\item
The guarantee is with respect to \math{\norm{\matX-\matX_{k}}_{\mathrm{F}}^2},
which is the best possible rank-\math{k} reconstruction with 
\math{k} dense optimal features obtained from PCA. Our result shows that
\math{O(k/\varepsilon)}-sparse features 
suffices to mimic top-\math{k} (dense) PCA 
within a \math{(1+\varepsilon)}-factor error.

We presented the simplified version of the results requiring
sparsity \math{r>5k}. 
Actually our technique can be applied for any choice of \math{r>k}
with approximation guarantee 
\math{c(1+1/(1-\sqrt{k/r})^2)} where \math{c\approx 1} is a constant.

\item
The reconstruction error using our \math{r}-sparse encoder is 
compared with PCA, not with the optimal \math{r}-sparse encoder.
With respect to the optimal \math{r}-sparse encoder, our approximation
could be much better.
It would be a challenging task to get an approximation 
guarantee with respect to the best possible \math{r}-sparse linear-encoder. 
\end{enumerate}

\subsection{Lower Bound on Sparsity}

We define the \emph{combined sparsity} of \math{\matH} to be the
number of its rows that are non-zero. When \math{k}=1, the combined 
sparsity equals the sparsity of the single factor. 
The combined sparsity is the
total number of dimensions which have non-zero loadings among
all the factors.
Our algorithm produces an encoder with combined sparsity 
\math{O(k/\varepsilon)} and comes within
\math{(1+\varepsilon)} of the minimum possible information loss.
We show that this is worst case optimal.
Specifically, there is a matrix \math{\matX} for which 
any linear encoder which can achieve a \math{(1+\varepsilon)}-approximate
reconstruction error as compared to PCA must have a combined 
sparsity
\math{r\ge k/\varepsilon}. So \math{\Omega(k/\varepsilon)}-sparsity is required
to achieve a \math{(1+\varepsilon)}-approximation.
The common case that is studied in the literature is with 
\math{k=1} (constructing a sparse top principal component).
Our lower bound shows that \math{\Omega(1/\varepsilon)}-sparsity is required
to get a \math{(1+\varepsilon)}-approximation
and our algorithm asymptotically achieves this lower bound, 
therefore we are asymptotically
optimal.

We show the converse of Theorem~\ref{theorem:blackbox}, namely that
from a linear auto-encoder with combined
sparsity \math{r}, we can construct 
\math{r} columns \math{\matC} for which \math{\matX_{\matC,k}} is a good
approximation to \math{\matX}. We then use the lower bound proved in 
\cite[Section 9.2]{malik195} which states that for 
\emph{any} \math{\delta>0}, there is a matrix 
\math{\matB} such that for any set of \math{r} of its columns
\math{\matC},
\mld{
\norm{\matB-\matB_{\matC,k}}_{\mathrm{F}}^2\ge
\norm{\matB-\matC\matC^\dagger\matB}_{\mathrm{F}}^2\ge
\left(1+\frac{k}{r}-\delta\right)\norm{\matB-\matB_{k}}_{\mathrm{F}}^2.
\label{eq:lower1}
}
Now suppose that \math{\matH} is a linear encoder with 
combined sparsity \math{r} for 
this matrix \math{\matB} from~\cite[Section 9.2]{malik195}, and
let 
\math{\matG} be a decoder satisfying
\mand{
\norm{\matB-\matB\matH\matG}_{\mathrm{F}}^2\le
(1+\varepsilon)\norm{\matB-\matB_k}_{\mathrm{F}}^2.
}
We offer the following elementary lemma which allows one to construct 
columns from a sparse encoder.
\begin{lemma}\label{lem:invbb}
Suppose \math{\matH} is a linear encoder for \math{\matB}
with combined sparsity \math{r} and
with decoder
\math{\matG}. Then,
\math{\matB\matH\matG=\matC\matY} for some set of \math{r} columns
of \math{\matB}, denoted \math{\matC},
 and some matrix \math{\matY\in\R^{r\times d}}.
\end{lemma}
\begin{proof}
Let \math{1\le i_1<i_2\cdots<i_r\le d} be the indices of the non-zero 
rows of \math{\matH} and let \math{\Omega=[\ee_{i_1},\ldots,\ee_{i_r}]} be
a sampling matrix for \math{\matB}. Every column of
\math{\matH} is in the subspace spanned by
the columns in~\math{\Omega}. Hence,
the projection of \math{\matH} onto the columns in \math{\Omega} gives back
\math{\matH}, that is \math{\matH=\Omega\Omega\transp\matH}. So,
\math{
\matB\matH\matG=\matB\Omega\Omega\transp\matH\matG=\matC\matY},
where \math{\matC=\matB\Omega} and \math{\matY=\Omega\transp\matH\matG}.
\end{proof}
By Lemma~\ref{lem:invbb}, and because 
\math{\norm{\matB-\matC\matC^\dagger\matB}_{\mathrm{F}}^2\le
\norm{\matB-\matC\matY}_{\mathrm{F}}^2} for any \math{\matY}, we have:
\mld{
\norm{\matB-\matC\matC^\dagger\matB}_{\mathrm{F}}^2
\le
\norm{\matB-\matC\matY}_{\mathrm{F}}^2
=
\norm{\matB-\matB\matH\matG}_{\mathrm{F}}^2\le
(1+\varepsilon)\norm{\matB-\matB_k}_{\mathrm{F}}^2.\label{eq:lower2}
}
It follows from \r{eq:lower1} and \r{eq:lower2} that 
\math{\varepsilon\ge k/r}.
Our algorithm gives a reconstruction error \math{\varepsilon=O(k/r)}, and
so our algorithm is asymptotically worst-case optimal.
The conclusion is that no linear encoder with combined
sparsity \math{r} can achieve an approximation ratio which is 
smaller than \math{(1+k/r)}. 
\begin{theorem}[Lower Bound on Sparsity]\label{theorem:lower}
There exists a data matrix \math{\matX} for which every \math{r}-sparse
encoder with \math{r<k/\epsilon} has an information loss which is 
strictly greater than \math{(1+\epsilon)\norm{\matX-\matX_k}_F^2}.
\end{theorem}
This theorem holds
for general linear auto-encoders, and so the lower bound also applies
to the symmetric auto-encoder \math{\matH\matH^\dagger}, the traditional
formulation of sparse PCA. For the case 
\math{k=1}, any \math{r}-sparse
unit norm \math{\vv},
\math{\norm{\matB-\matB\vv\vv\transp}_{\mathrm{F}}^2
\ge(1+\frac1r)\norm{\matB-\matB_1}_{\mathrm{F}}^2}, or the explained variance
(the variance in the residual) is \math{\vv\transp\matB\transp\matB\vv}
which 
is upper-bounded by
\mand{
\vv\transp\matB\transp\matB\vv\le \FNormS{\matB_1}-\frac{1}{r}
\FNormS{\matB-\matB_1}.}
Our algorithm achieves an explained variance which is 
\math{\FNormS{\matB_1}-\Omega(\frac{1}{r})
\FNormS{\matB-\matB_1}}.
Note that with respect to explained variance, the input
matrix 
\math{\matA=[\bm1,\bm1,\ldots,\bm1]} immediately gives an upper bound on the
approximation ratio of \math{(1-r/d)} for the top \math{r}-sparse PCA, since
every subset of \math{r} columns of \math{\matA} has spectral norm
\math{r} whereas the spectral norm of \math{\matA} is \math{d}.

\subsection{Iterative Sparse Linear Encoders}
Our previous result is strong in the sense that every vector (column) in the linear
encoder is \math{r}-sparse with 
non-zero loadings on the \emph{same} set of \math{r}
original feature dimensions.
From our lower bound, we know that the combined sparsity
of \math{\matH} 
must be at least \math{k/\varepsilon} to achieve reconstruction
error \math{(1+\varepsilon)\norm{\matX-\matX_k}} (in the worst case),
so our
algorithm from Section~\ref{sec:encoders}
is optimal in that sense. The algorithm is a ``batch'' algorithm in the sense
that, given \math{k}, it constructs all the \math{k} factors
in \math{\matH} simultaneously. We will refer to this algorithm as the
``batch algorithm''.
The batch algorithm may have non-zero loadings on 
all the \math{r} non-zero rows for every feature vector (column 
\math{\hh_i} of
\math{\matH}). 
Further, the batch algorithm does not
distinguish between the \math{k}-factors. That is, there is no 
top component, second component, and so on.

The traditional techniques for sparse PCA construct the factors iteratively.
We can run our batch algorithm in an iterative mode, where in each step we
set \math{k=1} and compute a sparse factor for a residual matrix.
By constructing our 
\math{k}
features iteratively (and adaptively), we identify an ordering
among the \math{k} features. Further, we might be
able to get each feature sparser while still maintaining a bound on
the
combined sparsity. The problem is that a straightforward
reduction of the iterative algorithm to CSSP is not possible. So 
we need to develop a new iterative version of the algorithm for which we 
can prove an approximation guarantee.
The iterative algorithm constructs each factor in the encoder 
sequentially, based on the
residual from the reconstruction using the previously constructed factors.
One advantage of the iterative algorithm is that it can control the
sparsity of each factor independently to achieve the desired approximation 
bound. Recall that the encoder \math{\matH=[\hh_1,\ldots,\hh_k]} is
\math{(r_1,\ldots,r_k)}-sparse if \math{\norm{\hh_i}_0\le r_i}.
The iterative algorithm is summarized below.
\begin{center}
\fbox{
\begin{minipage}{0.8\textwidth}
\underline{{\bf Iterative Sparse Linear Encoder Algorithm}}
\\[0.1in]
{\bf Inputs:} 
\math{\matX\in\R^{n\times d}}; target rank
\math{k\le\rank(\matX)}; sparsity parameters \math{r_1,\ldots,r_k}.

{\bf Output:} \math{(r_1,\ldots,r_k)}-sparse linear encoder 
\begin{algorithmic}[1]
\STATE Set the residual \math{\Delta=\matX} and \math{\matH=[\ ]}.
\FOR{\math{i=1} to \math{k}}
\STATE Compute an encoder \math{\hh} for \math{\Delta} using the batch algorithm
with \math{k=1} and sparsity parameter \math{r=r_i}.
\STATE Update the encoder by adding \math{\hh} to it:
\math{\matH\gets[\matH,\hh]}.
\STATE Update the residual \math{\Delta}: \math{\Delta\gets
\matX-\matX\matH(\matX\matH)^\dagger\matX}.
\ENDFOR
\STATE Return the \math{(r_1,\ldots,r_k)}-sparse encoder 
\math{\matH\in\R^{n\times k}}.
\end{algorithmic}
\end{minipage}
}
\end{center}
The main iterative step in the algorithm is occuring in steps 3,4 above where
we augment \math{\matH} by computing a top sparse encoder for the
residual obtained from the current \math{\matH}. 
The next lemma bounds the reconstruction error for this iterative step
in the algorithm.
\begin{lemma}\label{lem:iterative}
Suppose, for \math{k\ge1},
 \math{\matH_k=[\hh_1,\ldots,\hh_k]} is an encoder for 
\math{\matX}, satisfying
\mand{\norm{\matX-\matX\matH_k(\matX\matH_k)^\dagger\matX}_{\mathrm{F}}^2
= {\sf err}.} 
Given a sparsity \math{r>5}, one can compute in time
\math{O(ndr+(n+d)r^2)} 
an \math{r}-sparse feature vector 
\math{\hh_{k+1}} such that for the encoder 
\math{\matH_{k+1}=[\hh_1,\ldots,\hh_k,\hh_{k+1}]}, the reconstruction error
satisfies
\mand{\Exp\left[\norm{\matX-\matX\matH_{k+1}(\matX\matH_{k+1})^\dagger\matX}_{\mathrm{F}}^2
\right]
= (1+\delta)({\sf err}-
\norm{\matX-\matX\matH_{k}(\matX\matH_{k})^\dagger\matX}_2^2),
}
where \math{\delta\le 5/(r-5)}.
\end{lemma}
\begin{proof}
Let \math{\matG_k=(\matX\matH_k)^\dagger\matX}
and \math{\matB=\matX-\matX\matH_k\matG_k}. We have that
\math{\norm{\matB}_{\mathrm{F}}^2={\sf err}.} Run our batch algorirhm from 
Theorem~\ref{theorem:SPCA} on \math{\matB} 
with \math{k=1} and sparsity parameter 
\math{r} to obtain an \math{r}-sparse encoder
\math{\hh_{k+1}} and corresponding decoder \math{\gg_{k+1}} satisfying
\mand{
\Exp\left[\norm{\matB-\matB\hh_{k+1}\gg_{k+1}}_{\mathrm{F}}^2\right]=
(1+\delta)\norm{\matB-\matB_1}_{\mathrm{F}}^2,}
with \math{\delta\le 5/(r-5)}.
Now,  the RHS is \math{\norm{\matB-\matB_1}_{\mathrm{F}}^2=\norm{B}_{\mathrm{F}}^2-\norm{B}_2^2
={\sf err}-\norm{B}_2^2}. The LHS is
\eqan{
\Exp\left[\norm{\matB-\matB\hh_{k+1}\gg_{k+1}\transp}_{\mathrm{F}}^2\right]
&=&
\Exp\left[\norm{\matX-\matX\matH_k\matG_k-\matX\hh_{k+1}\gg_{k+1}\transp+
\matX\matH_k\matG_k\hh_{k+1}\gg_{k+1}\transp}_{\mathrm{F}}^2\right]\\
&=&
\Exp\left[\norm{\matX-\matX\matH_k(\matG_k-
\matG_k\hh_{k+1}\gg_{k+1}\transp)-\matX\hh_{k+1}\gg_{k+1}\transp}_{\mathrm{F}}^2\right]\\
&=&
\Exp\left[\norm{\matX-\matX\matH_{k+1}\matG_{k+1}}_{\mathrm{F}}^2\right],
}
where \math{\matG_{k+1}=[\matG_k-\gg_{k+1}\hh_{k+1}\transp\matG_{k},\gg_{k+1}]}.
That is, \math{\matH_{k+1}} with the decoder \math{\matG_{k+1}} satisfies the
expected error bound in the lemma, so the optimal decoder cannot do worse 
in expectation.
\end{proof}
We note that the iterative algorithm produces an 
encoder \math{\matH} which is not guaranteed to be orthonormal.
This is not critical for us, since the information loss of the encoder
is based on the minimum possible reconstruction error, and we can 
compute this information loss even if the encoder is not orthonormal.
Observe also that, if so desired, 
it is possible to orthonormalize the columns of
\math{\matH} without changing the combined sparsity.

Lemma~\ref{lem:iterative} gives a bound on the reconstruction
error for an iterative addition of the next sparse encoder vector. 
As an example of
how we apply Lemma~\ref{lem:iterative},
\remove{%
as follows. Suppose that
\math{\matH_k} gives \math{(1+\varepsilon)\norm{\matX-\matX_k}_{\mathrm{F}}^2} 
reconstruction error
and let \math{\hat\matX=\matX\matH_{k+1}(\matX\matH_{k+1})^\dagger\matX} be the
reconstruction from \math{\matH_{k+1}} as in 
Lemma~\ref{lem:iterative}. Then,
setting \math{\delta=\varepsilon} in Lemma~\ref{lem:iterative}, that is
\math{r=5+5/\varepsilon}, we have that
\eqan{\Exp\left[\norm{\matX-\hat\matX}_{\mathrm{F}}^2
\right]
&\le&(1+\varepsilon)\left((1+\varepsilon)\norm{\matX-\matX_k}_{\mathrm{F}}^2-
\norm{\matX-\hat\matX}_2^2\right)\\
&=&
(1+\varepsilon)\left((1+\varepsilon)\norm{\matX-\matX_{k+1}}_{\mathrm{F}}^2
+(1+\varepsilon)\sigma_{k+1}^2-\norm{\matX-\hat\matX}_2^2\right).}
Since \math{\hat\matX} has rank 
at most \math{k}, by the Eckart-Young Theorem, we get a crude bound
\math{
\norm{\matX-\hat\matX}_2^2
\ge
\norm{\matX-\matX_k}_2^2=\sigma_{k+1}^2}.
We therefore have that
\mand{
\Exp\left[\norm{\matX-\hat\matX}_{\mathrm{F}}^2
\right]
\le
(1+\varepsilon)^2\norm{\matX-\matX_{k+1}}_{\mathrm{F}}^2+\varepsilon(1+\varepsilon)
\norm{\matX-\matX_{k}}_2^2.
}
The bound is a relative error bound on the reconstruction error
up to a small additive term \math{\varepsilon(1+\varepsilon)
\norm{\matX-\matX_{k}}_2^2}. \christos{the discussion in this paragraph should go either inside the previous lemma or be a separate lemma. no reason to be a discussion. but this is not the interesting result; i think the interesting result is to show what happens if we start applying this from $\ell = 1:k$. what is the final bound?}
}%
suppose 
the target rank is \math{k=2}. We start by constructing
\math{\hh_1} with sparsity \math{r_1=5+5/\varepsilon}, which gives us
\math{(1+\varepsilon)\norm{\matX-\matX_1}_{\mathrm{F}}^2} reconstruction error.
We now construct \math{\hh_2}, also with sparsity \math{r_2=5+5/\varepsilon}. 
The
final reconstruction error for 
\math{\matH=[\hh_1,\hh_2]} is bounded by 
\mand{
\norm{\matX-\matX\matH(\matX\matH)^\dagger\matX}_{\mathrm{F}}^2\le
(1+\varepsilon)^2\norm{\matX-\matX_{2}}_{\mathrm{F}}^2+\varepsilon(1+\varepsilon)
\norm{\matX-\matX_{1}}_2^2.
}
On the other hand, our batch algorithm uses sparsity
\math{r=10+10/\varepsilon} in each encoder \math{\hh_1,\hh_2} and achieves
reconstruction error \math{(1+\varepsilon)\norm{\matX-\matX_2}_{\mathrm{F}}^2}.
The iterative algorithm uses sparser features, but pays for it a little
in reconstruction error. 
The additive term is small: it
is \math{O(\varepsilon)} and depends on 
\math{\norm{\matX-\matX_{1}}_2^2=\sigma_2^2}, which in practice
is smaller than \math{\norm{\matX-\matX_{2}}_{\mathrm{F}}^2=\sigma_3^2+\cdots+\sigma_d^2}.
In practice, though the theoretical bound 
for the iterative algorithm
is slightly worse than the batch algorithm guarantee, 
the iterative algorithm performs comparably to the batch algorithm, 
it is more flexible and able to generate 
encoder vectors that are 
sparser than the batch algorithm.

Using the iterative algorithm, we can tailor the sparsity of each
encoder vector separately to achieve a desired accuracy. It is algebraically
intense to prove a bound for a general choice of the sparsity parameters
\math{r_1,\ldots,r_k}, so, for simplicity, we prove a bound for a
specific choice of the 
sparsity parameters which slowly increase for each additional
encoder vector. We get the following result.
\begin{theorem}[Adaptive iterative encoder]\label{theorem:mainIterative}
Given \math{\matX\in\R^{n\times d}} of rank \math{\rho} and 
\math{k<\rho}, there is an algorithm to compute encoder vectors 
\math{\hh_1,\hh_2,\ldots,\hh_k} iteratively with each encoder vector
\math{\hh_j} having sparsity \math{r_j=5+\ceil{5j/\varepsilon}} such that
for every \math{\ell=1,\ldots,k}, the encoder 
\math{\matH_\ell=[\hh_1,\hh_2,\ldots,\hh_\ell]} has information loss:
\mld{
\Exp\left[\norm{\matX-\matX\matH_\ell(\matX\matH_\ell)^\dagger\matX}_{\mathrm{F}}^2\right]
\le
(e\ell)^\varepsilon\norm{\matX-\matX_\ell}_{\mathrm{F}}^2
+\varepsilon \ell^{1+\varepsilon}\norm{\matX_\ell-\matX_1}_{\mathrm{F}}^2.\label{eqthm:mainIterative}
}
The running time to compute all the
encoder vectors is \math{O(ndk^2\varepsilon^{-1}+(n+d)k^3\varepsilon^{-2})}.
\end{theorem}
\paragraph{Comments.}
\begin{enumerate}[1.]
\item \math{(e\ell)^{\varepsilon}} 
is a very slowly growing in \math{\ell}. 
For example, for \math{\varepsilon=0.01} and \math{\ell=100},
\math{\ell^{\varepsilon}\approx1.04}.
Asymptotically,  
\math{(e\ell)^{\varepsilon}=1+O(\epsilon\log\ell)}, so up to a small 
additive term, we have a relative error approximation. The message is that
we get a reasonable approximation guarantee for our iterative algorithm,
where no such bounds are available for existing algorithms (which are
all iterative). 

\item Observe that each successive encoder vector has a larger value
of the sparsity parameter \math{r_j}. 
In the batch algorithm, every 
encoder vector has sparsity parameter \math{r=5k+5k/\varepsilon} to get a
reconstruction error of \math{(1+\varepsilon)\norm{\matX-\matX_k}_{\mathrm{F}}^2}. This
means that the number of non-zeros in \math{\matH} is at most
\math{5k^2+5k^2/\varepsilon}. For the iterative encoder, the first 
few encoder vectors are very sparse, getting denser
as you add more encoder vectors until you get to 
the last encoder vector which has
maximum sparsity parameter \math{r_k=5+5k/\varepsilon}. 
The tradeoff is in the reconstruction error.
We can get still sparser in the iterative algorithm, setting every 
encoder vector's sparsity to \math{r_j=5+5k/\varepsilon}. In this case,
the bound on the reconstruction error for \math{\matH_\ell} becomes
\math{(1+\varepsilon)^\ell\norm{\matX-\matX_\ell}_{\mathrm{F}}^2+
\frac12 \varepsilon(1+\varepsilon)^{\ell-1}\ell\norm{\matX_\ell-\matX_1}_{\mathrm{F}}^2},
which is growing as \math{1+O(\varepsilon\ell)} as opposed to 
\math{1+O(\varepsilon\log\ell)}.
One other trade off with the iterative algorithm is that
the combined sparsity (number of non-zero rows) of \math{\matH} 
could increase, as compared to the batch algorithm.
In the batch algorithm, the combined sparsity is \math{O(k/\varepsilon)},
the same as the sparsity parameter of each encoder vector,
since every encoder vector has non-zeros in the same set of rows.
With the iterative algorithm, every encoder vector could conceivably have
non-zeros in different rows giving \math{O(k^2/\varepsilon)} non-zero rows.

\item 
Just as with the PCA vectors \math{\vv_1,\vv_2,\ldots,}, we have an
encoder for any choice of \math{\ell} by taking the first \math{\ell} encoder
vectors \math{\hh_1,\ldots,\hh_\ell}. This is not the case for the 
batch algorithm. If we compute the batch-encoder 
\math{\matH=[\hh_1,\ldots,\hh_k]}, we 
cannot guarantee that the first encoder vector \math{\hh_1} will give
a good reconstruction comparable with \math{\matX_1}.
\end{enumerate}

\begin{proof}(Theorem~\ref{theorem:mainIterative})
For \math{\ell\ge1}, we define two quantities \math{Q_\ell,P_\ell} for that will be useful in the proof.
\eqan{
Q_\ell&=&\textstyle(1+\varepsilon)(1+\frac12\varepsilon)(1+\frac13\varepsilon)(1+\frac14\varepsilon)
\cdots(1+\frac1\ell\varepsilon);\\
P_\ell
&=&\textstyle(1+\varepsilon)(1+\frac12\varepsilon)(1+\frac13\varepsilon)(1+\frac14\varepsilon)
\cdots(1+\frac1\ell\varepsilon)-1
=Q_\ell-1.}
Using Lemma~\ref{lem:iterative} and induction, 
we prove a bound
on the information loss of encoder \math{\matH_\ell}:
\mld{
\Exp\left[\norm{\matX-\matX\matH_\ell(\matX\matH_\ell)^\dagger\matX}_{\mathrm{F}}^2\right]
\le
Q_\ell\norm{\matX-\matX_\ell}_{\mathrm{F}}^2+
Q_\ell\sum_{j=2}^\ell\sigma_j^2\frac{P_{j-1}}{Q_{j-1}}.
\tag{\math{*}}\label{eq:proof1}
}
When \math{\ell=1}, the claim is that
\math{\Exp[\norm{\matX-\matX\matH_1(\matX\matH_1)^\dagger\matX}_{\mathrm{F}}^2]
\le
(1+\varepsilon)\norm{\matX-\matX_1}_{\mathrm{F}}^2} (since the summation is empty),
which is true by construction of \math{\matH_1=[\hh_1]} because
\math{r_1=5+5/\varepsilon}.
Suppose the claim in \r{eq:proof1} holds up to \math{\ell\ge1} and consider
\math{\matH_{\ell+1}=[\matH_\ell,\hh_{\ell+1}]}, where \math{\hh_{\ell+1}} has
sparsity \math{r_{\ell+1}=5+5(\ell+1)/\varepsilon}. We apply 
Lemma~\ref{lem:iterative} with \math{\delta=\varepsilon/(\ell+1)} and 
we condition on
 \math{{\sf err}=\norm{\matX-\matX\matH_\ell(\matX\matH_\ell)^\dagger\matX}_{\mathrm{F}}^2} whose  expectation is given  in \r{eq:proof1}. 
By iterated expectation, we have that
\eqan{
&&\Exp\left[\norm{\matX-\matX\matH_{\ell+1}(\matX\matH_{\ell+1})^\dagger\matX}_{\mathrm{F}}^2
\right]\\
&\buildrel(a)\over=&
\Exp_{\matH_\ell}\Exp_{\hh_{\ell+1}}\left[\norm{\matX-\matX\matH_{\ell+1}(\matX\matH_{\ell+1})^\dagger\matX}_{\mathrm{F}}^2\mid\matH_\ell
\right]\\
&\buildrel(b)\over\le&\left(1+\frac{\varepsilon}{\ell+1}\right)
\Exp_{\matH_\ell}\left[\norm{\matX-\matX\matH_{\ell}(\matX\matH_{\ell})^\dagger\matX}_{\mathrm{F}}^2-\norm{\matX-\matX\matH_{\ell}(\matX\matH_{\ell})^\dagger\matX}_2^2\right]
\\
&\buildrel(c)\over\le&
\frac{Q_{\ell+1}}{Q_{\ell}}\left(Q_\ell\norm{\matX-\matX_\ell}_{\mathrm{F}}^2+
Q_\ell\sum_{j=2}^\ell\sigma_j^2\frac{P_{j-1}}{Q_{j-1}}-
\underbrace{\Exp_{\matH_\ell}\left[\norm{\matX-\matX\matH_{\ell}(\matX\matH_{\ell})^\dagger\matX}_2^2\right]}_{\ge
\sigma_{\ell+1}^2}\right)\\
&\le&
\frac{Q_{\ell+1}}{Q_{\ell}}\left(Q_\ell\norm{\matX-\matX_\ell}_{\mathrm{F}}^2-\sigma_{\ell+1}^2+
Q_\ell\sum_{j=2}^\ell\sigma_j^2\frac{P_{j-1}}{Q_{j-1}}\right)\\
&=&
\frac{Q_{\ell+1}}{Q_{\ell}}\left(Q_\ell\left(\sigma_{\ell+1}^2+\norm{\matX-\matX_{\ell+1}}_{\mathrm{F}}^2\right)
-
\sigma_{\ell+1}^2+
Q_\ell\sum_{j=2}^\ell\sigma_j^2\frac{P_{j-1}}{Q_{j-1}}\right)\\
&=& 
\frac{Q_{\ell+1}}{Q_{\ell}}\left(Q_\ell\norm{\matX-\matX_{\ell+1}}_{\mathrm{F}}^2
+
\sigma_{\ell+1}^2(Q_{\ell}-1)+
Q_\ell\sum_{j=2}^\ell\sigma_j^2\frac{P_{j-1}}{Q_{j-1}}\right)\\
&\buildrel(d)\over=&
Q_{\ell+1}\norm{\matX-\matX_{\ell+1}}_{\mathrm{F}}^2
+
\sigma_{\ell+1}^2\frac{Q_{\ell+1}P_\ell}{Q_{\ell}}+
Q_{\ell+1}\sum_{j=2}^\ell\sigma_j^2\frac{P_{j-1}}{Q_{j-1}}\\
&=&
Q_{\ell+1}\norm{\matX-\matX_{\ell+1}}_{\mathrm{F}}^2
+
Q_{\ell+1}\sum_{j=2}^{\ell+1}\sigma_j^2\frac{P_{j-1}}{Q_{j-1}}.
}
In (a) we used iterated expectation. In (b) we used Lemma~\ref{lem:iterative}
to take the expectation over \math{\hh_{\ell+1}}. In (c), we used
the definition of \math{Q_\ell} from which
\math{Q_{\ell+1}/Q_\ell=(1+\varepsilon/(\ell+1))}, and the induction hypothesis
to bound 
\math{\Exp_{\matH_\ell}[\norm{\matX-
\matX\matH_{\ell}(\matX\matH_{\ell})^\dagger\matX}_{\mathrm{F}}^2]}. We also observed that,
since \math{\matX\matH_{\ell}(\matX\matH_{\ell})^\dagger\matX} is a rank-\math{k}
approximation to \math{\matX}, it follows from the Eckart-Young theorem that
\math{\norm{\matX-\matX\matH_{\ell}(\matX\matH_{\ell})^\dagger\matX}_2^2\ge\sigma_{\ell+1}^2} for \emph{any} \math{\matH_\ell}, and hence this inequality also holds
in expectation.
In (d) we used the definition \math{P_\ell=Q_\ell-1}.
The bound \r{eq:proof1} now follows by induction for \math{\ell\ge1}.
The first term in the bound \r{eqthm:mainIterative} follows by bounding
\math{Q_{\ell}} using elementary calculus:
\mand{
\log Q_\ell
=\sum_{i=1}^\ell\log\left(1+\frac{\varepsilon}{i}\right)\\
\le\sum_{i=1}^\ell\frac{\varepsilon}{i}\\
\le\varepsilon\log(e\ell),
}
where we used \math{\log(1+x)\le x} for \math{x\ge0} and the well known
upper bound \math{\log(e\ell)} for the \math{\ell}th harmonic number
\math{1+\frac12+\frac13+\cdots+\frac1\ell}. Thus,
\math{Q_\ell\le(e\ell)^\varepsilon}. The rest of the proof is to bound the second
term in \r{eq:proof1} to obtain the second term in 
\r{eqthm:mainIterative}. Obeserve that for \math{i\ge1},
\mand{
P_i=Q_i-1=\varepsilon\frac{Q_i}{Q_1}+\frac{Q_i}{Q_1}-1
\le\varepsilon\frac{Q_i}{Q_1}+Q_{i-1}-1=\varepsilon\frac{Q_i}{Q_1}+P_{i-1},
}
where we used \math{{Q_i}/{Q_1}\le Q_{i-1}} and we define 
\math{P_{0}=0}.
Therefore, 
\eqar{
\sum_{j=2}^\ell\sigma_j^2\frac{P_{j-1}}{Q_{j-1}}
&\le&
\frac{\varepsilon}{Q_1}\sum_{j=2}^\ell\sigma_j^2
+\sum_{j=3}^\ell\sigma_j^2\frac{P_{j-2}}{Q_{j-1}}\nonumber\\
&=&
\frac{\varepsilon}{Q_1}\sum_{j=2}^\ell\sigma_j^2
+\sum_{j=3}^\ell\sigma_j^2\frac{P_{j-2}}{Q_{j-2}}\cdot\frac{Q_{j-2}}{Q_{j-1}}
\nonumber\\
&\buildrel(a)\over\le&
\frac{\varepsilon}{Q_1}\sum_{j=2}^\ell\sigma_j^2
+\sum_{j=3}^\ell\sigma_j^2\frac{P_{j-2}}{Q_{j-2}},\nonumber\\
&=&
\frac{\varepsilon}{Q_1}\norm{\matX_\ell-\matX_1}_{\mathrm{F}}^2
+\sum_{j=3}^\ell\sigma_j^2\frac{P_{j-2}}{Q_{j-2}}.\nonumber
}
In (a) we used  \math{{Q_{j-2}}/{Q_{j-1}}<1}. The previous derivation
gives a reduction from which 
it is now
an elementary task to prove by induction that
\mand{
 \sum_{j=2}^\ell\sigma_j^2\frac{P_{j-1}}{Q_{j-1}}
\le
\frac{\varepsilon}{Q_1}\sum_{j=1}^{\ell-1}\norm{\matX_\ell-\matX_j}_{\mathrm{F}}^2.
}
Since \math{\norm{\matX_\ell-\matX_j}_{\mathrm{F}}^2\le\norm{\matX_\ell-\matX_1}_{\mathrm{F}}^2(\ell-j)/(\ell-1)}, we have that
\mand{
 \sum_{j=2}^\ell\sigma_j^2\frac{P_{j-1}}{Q_{j-1}}
\le
\frac{\varepsilon\norm{\matX_\ell-\matX_1}_{\mathrm{F}}^2}{Q_1(\ell-1)}\sum_{j=1}^{\ell-1}\ell-j
=
\frac{\varepsilon\ell\norm{\matX_\ell-\matX_1}_{\mathrm{F}}^2}{2Q_1}.
}
Using \r{eq:proof1}, we have that
\mand{
\norm{\matX-\matX\matH_\ell(\matX\matH_\ell)^\dagger\matX}_{\mathrm{F}}^2
\le
(e\ell)^\varepsilon\norm{\matX-\matX_\ell}_{\mathrm{F}}^2+
\frac{\varepsilon\ell\norm{\matX_\ell-\matX_1}_{\mathrm{F}}^2}{2}\cdot\frac{Q_\ell}{Q_1}.
}
The result finally follows because
\mand{
\log\frac{Q_\ell}{Q_1}
=\sum_{i=2}^\ell\log\left(1+\frac{\varepsilon}{i}\right)
\le\varepsilon\sum_{i=2}^\ell\frac1i\le\varepsilon(\log(e\ell)-1)=\varepsilon\log\ell,
}
and so \math{Q_\ell/Q_1\le\ell^{\varepsilon}}.
\end{proof}

\section{Experiments}
We compare the empirical performance of our algorithms 
with 
some of the existing state-of-the-art 
sparse PCA methods. 
The inputs are 
$\matX \in \R^{n \times d}$, the number of components
$k$
and the sparsity parameter $r$.
The output
is the sparse
encoder \math{\matH=[\hh_1, \hh_2, \dots, \hh_k] \in \R^{n \times k}} 
with 
$\norm{\hh_i}_0\le r$; 
\math{\matH} is used to project \math{\matX} onto some
subspace to obtain a reconstruction \math{\hat\matX} which decomposes
the variance into two terms:
\eqan{
\FNormS{\matX}
&=&
\FNormS{\matX-\hat\matX}+\FNormS{\hat\matX}\\
&=&\text{Reconstruction Error} + \text{Explained Variance}
}
Previous methods construct the matrix $\matH$ 
to have orthonormal columns and restristed the resonstruction to symmetric
auto-encoders,  $\hat\matX=\FNormS{\matX \matH \matH^\dagger}$.
Minimizing the reconstruction error is equivalent to 
maximizing the explained variance, so one metric that we consider
is the (normalized)
\emph{explained variance of the symmetric auto-encoder},
$$
\text{Symmetric Explained Variance}=
\frac{\FNormS{\matX \matH \matH^\dagger}}{\FNormS{\matX_k}}\le 1
$$
To capture how informative
the 
sparse components are, we use the normalized
information loss:
$$
\text{Information Loss}=
\frac{\FNormS{\matX - \matX \matH (\matX \matH)^\dagger\matX}}{\FNormS{\matX
-\matX_k}}\ge 1.
$$
The true explained variance when using the optimal decoder 
will be larger than the symmetric explained variance because
$\FNormS{\matX - \matX \matH \matH^\dagger} 
\ge \FNormS{\matX - \matX \matH (\matX \matH)^\dagger\matX}$. We report
the symmetric explained variance primarily for 
historical reasons because existing sparse PCA methods have constructed
auto-encoders to optimize the symmetric explained variance rather than the
true explained variance.
\remove{
such that it explains as much of the variance of 
$\matX$ as possible, i.e., $\FNormS{\matX \matH \matH^\dagger}$ is as large as possible. From the matrix Pythagoras theorem we have: 
$\FNormS{\matX \matH \matH^\dagger} = \FNormS{\matX} - \FNormS{\matX - \matX \matH \matH^\dagger},$ hence maximizing the explained variance 
$\FNormS{\matX \matH \matH^\dagger}$ is equivalent to minimizing the reconstruction error 
$\FNormS{\matX - \matX \matH \matH^\dagger}$. Notice, however, that we also have
$\FNormS{\matX - \matX \matH \matH^\dagger} \ge \FNormS{\matX - \matX \matH (\matX \matH)^\dagger\matX}$. 
On the other hand our sparse auto-encoder methods are designed to directly minimize the reconstruction error 
$\FNormS{\matX - \matX \matH (\matX \matH)^\dagger\matX}$. 
Those objectives are approximately equivalent, hence to compare our algorithms with previous sparse PCA methods we report results on both of them. The reconstruction error objective is exactly the auto-encoder framework we have discussed in the paper, applied to a symmetric PSD matrix. 
Let $\matU_k \in \R^{n \times k}$ contain the eigenvectors of $\matX$ corresponding to the $k$ largest eigenvalues of $\matX$. 
Notice that from standard linear algebra arguments we have that
$$
\matU_k = \arg\max_{\matZ \in \R^{n \times k} }\FNormS{\matX \matZ \matZ^\dagger}; \hspace{.5in}
\matX_k = \matX \matU_k \matU_k\transp = \arg\min_{\matB \in \R^{n \times n}, \rank(\matB) \le k} \FNormS{\matX - \matB}
$$ 
In our experiments, we report the
$ReconstructionError = \FNormS{\matX - \matX\matH (\matX \matH)^\dagger\matX} / \FNormS{\matX - \matX \matU_k \matU_k\transp} \ge 1$, and the $ExplainedVariance = \FNormS{\matX\matH \matH^\dagger} / \FNormS{\matX \matU_k \matU_k\transp} \le 1$ - both are the 	``normalized'' values. 
}

\begin{figure*}[t]
\begin{center}
\begin{tabular}{ccc}
\resizebox{0.3\textwidth}{!}{\includegraphics*{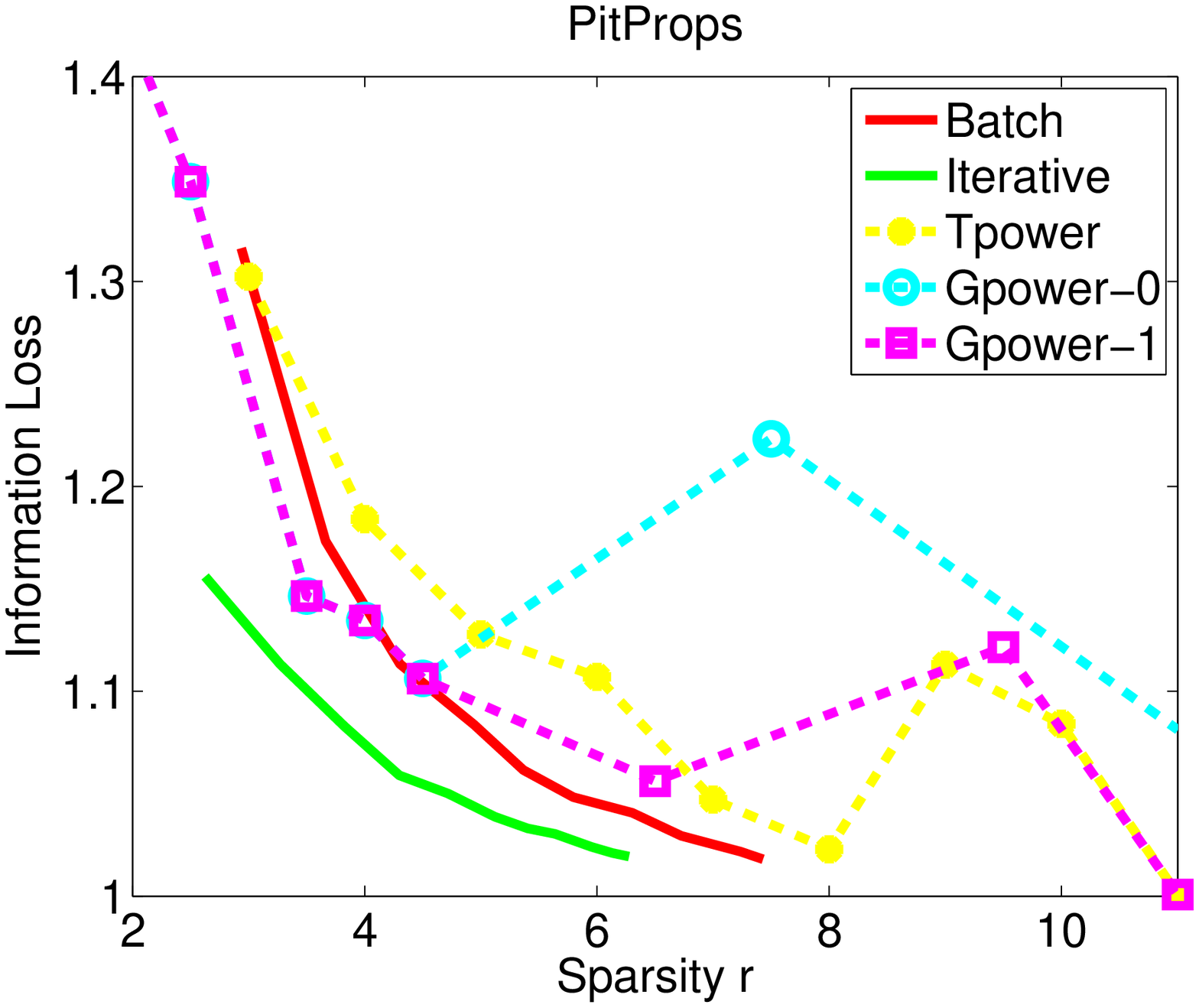}}
&
\resizebox{0.3\textwidth}{!}{\includegraphics*{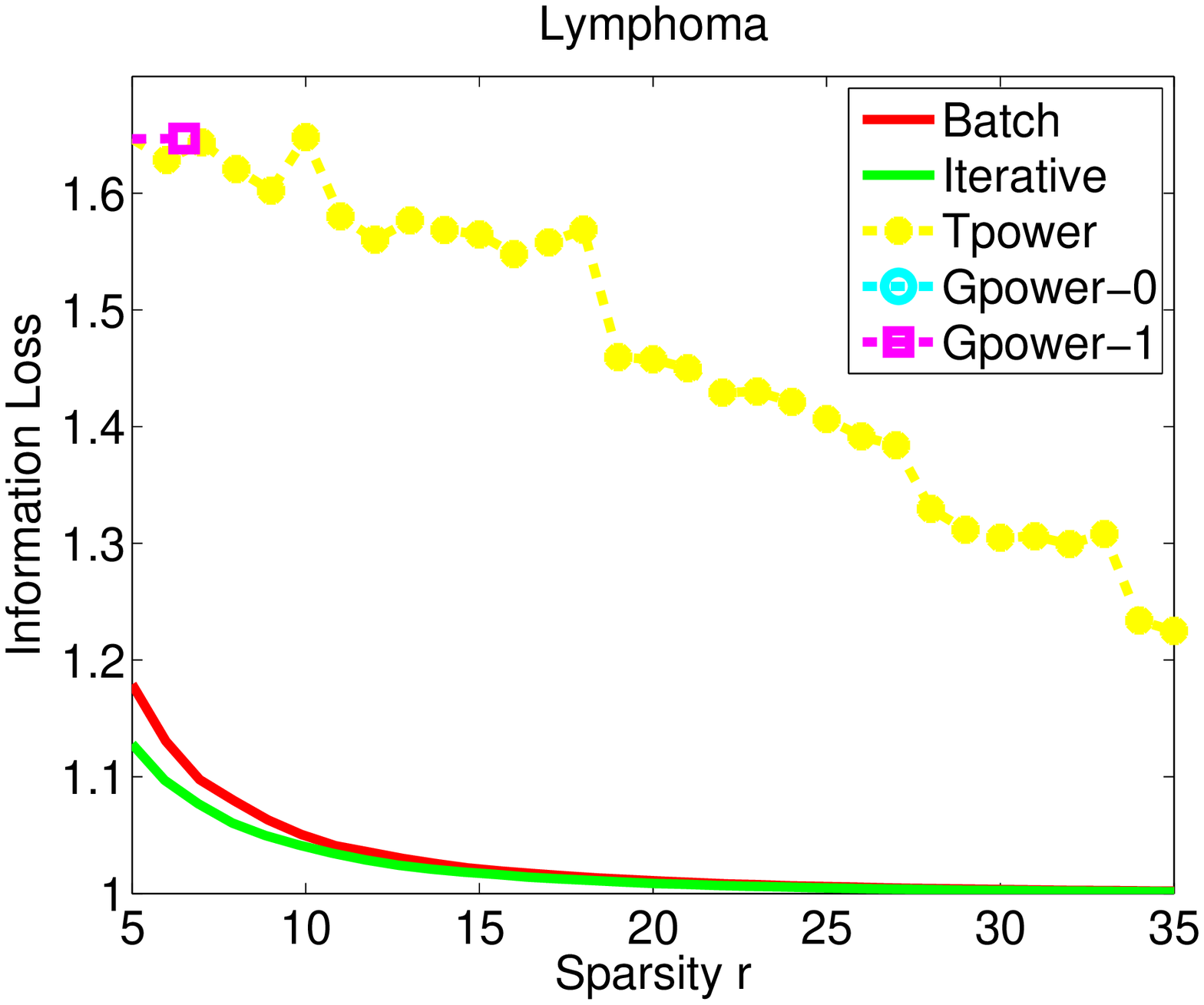}}
&
\resizebox{0.3\textwidth}{!}{\includegraphics*{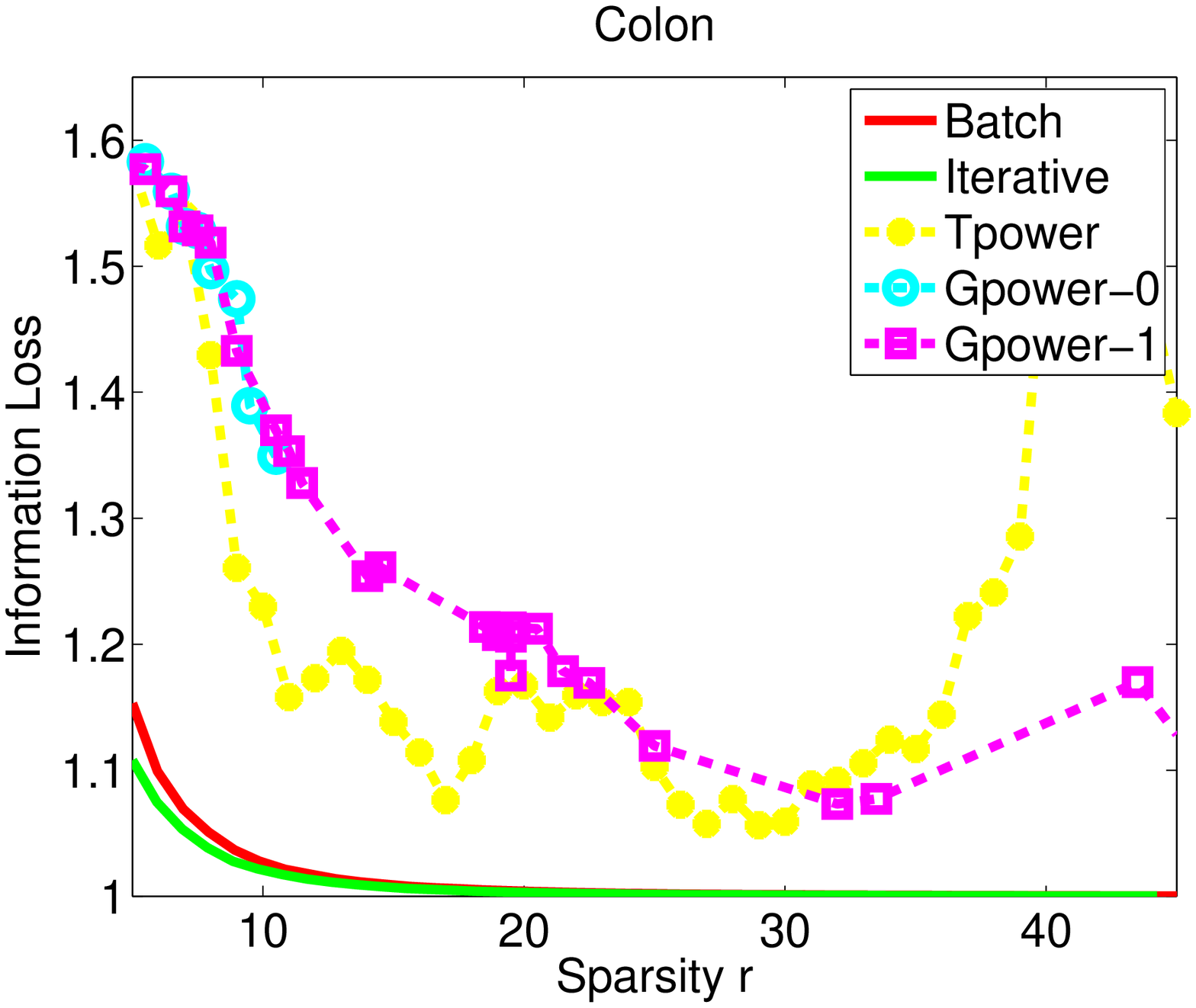}}
\\
\resizebox{0.3\textwidth}{!}{\includegraphics*{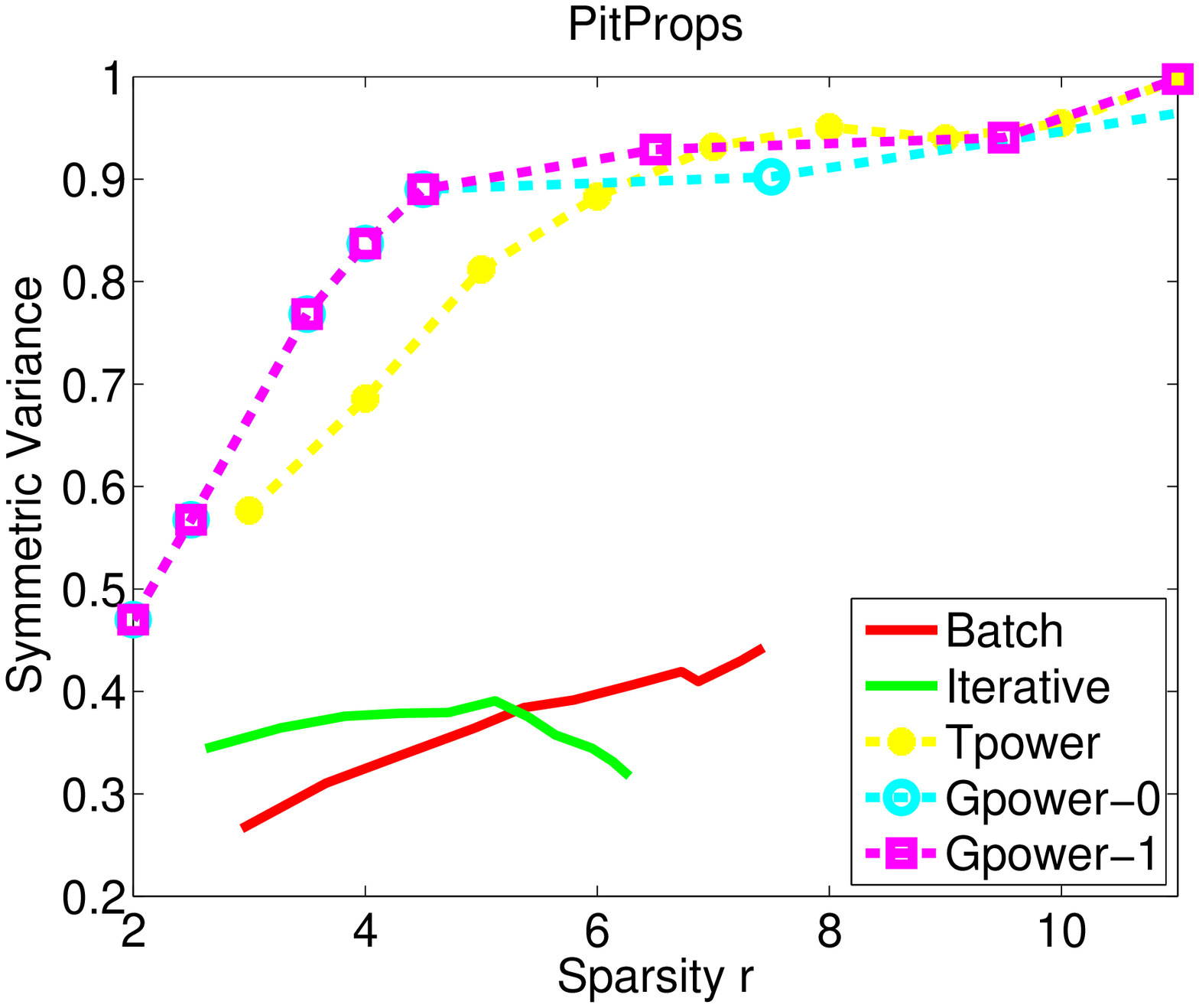}}
&
\resizebox{0.3\textwidth}{!}{\includegraphics*{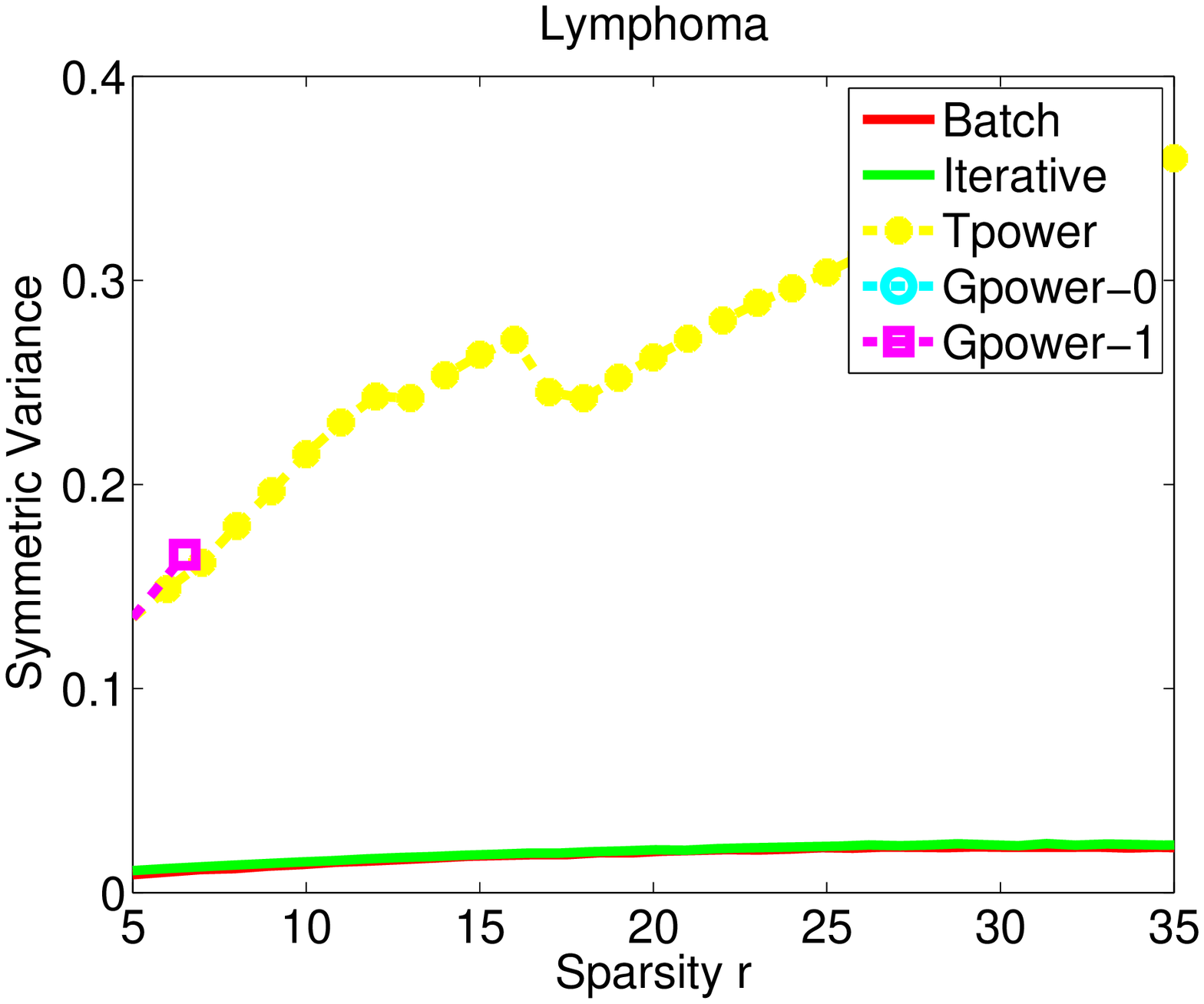}}
&
\resizebox{0.3\textwidth}{!}{\includegraphics*{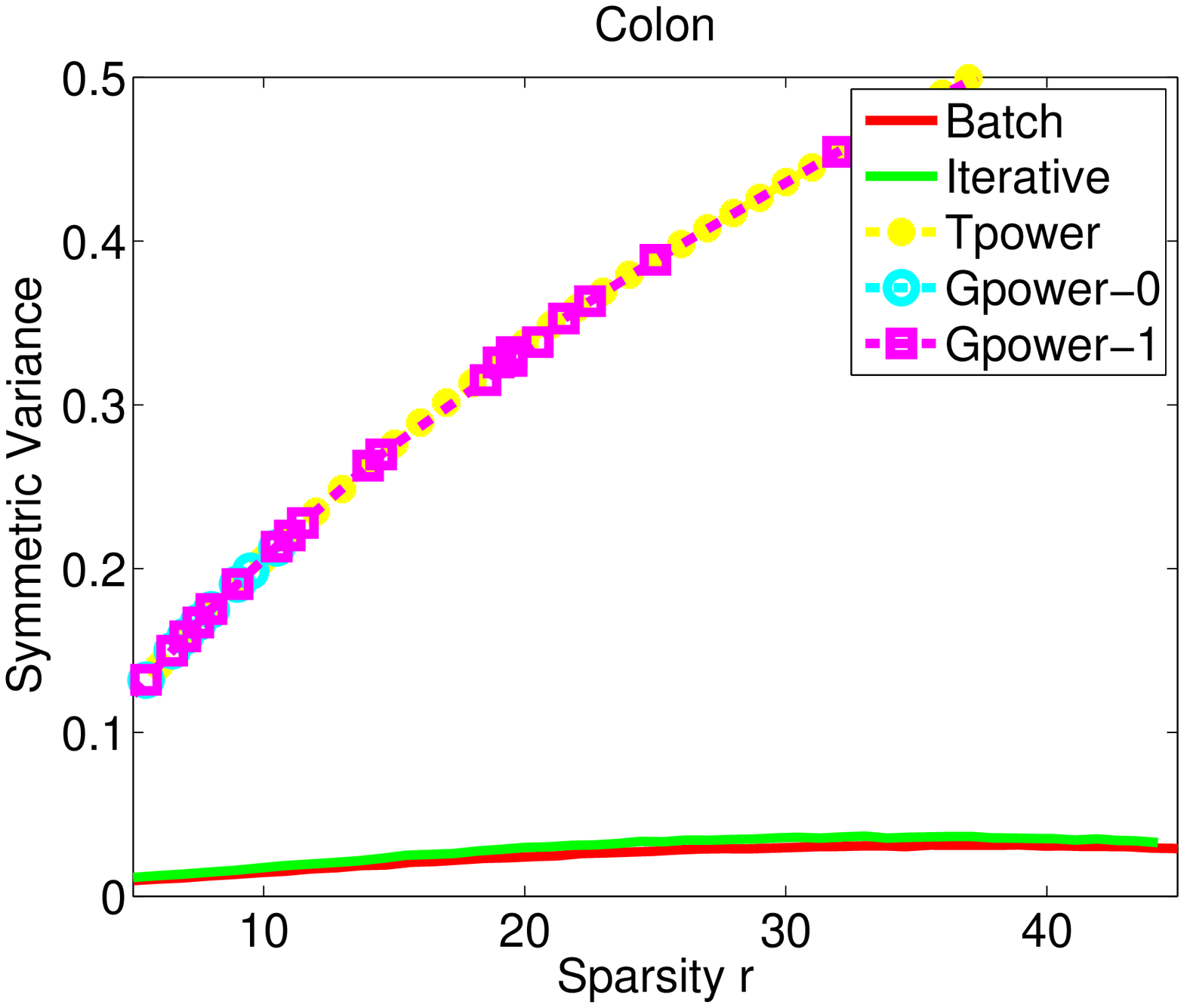}}
\end{tabular}
\end{center}
\caption{Performance of the sparse encoder algorithms
on the PitProps data (left),  Lymphoma data (middle) and Colon data (right) 
data:
The figures show the 
Information loss (top) and symmetric explained variance (bottom) with
\math{k=2}. We can observe that our algorithms give the 
best information loss which appears to be decreasing inversely with \math{r}
as the theory predicts. Existing sparse PCA algorithms which maximize
symmetric explained variance, not surprisingly, perform better with
respect to symmetric explained variance. The figures highlight that
information loss and symmetric explained variance are quite
different metrics. We argue that information loss is the meaningful
criterion to optimize.
\label{fig:main1}}
\end{figure*}

\subsection{Algorithms}
We implemented the following variant of the sparse PCA algorithm of Theorem~\ref{theorem:SPCA}: in the general framework of the algorithm described in Section~\ref{sec:encoders}, we use the deterministic technique described in part $(i)$ in Theorem~\ref{theorem:Cmain} in order to find the matrix $\matC$ with $r$ columns of $\matX$. (This algorithm gives a constant factor approximation, as opposed to the relative error approximation of the algorithm in Theorem~\ref{theorem:SPCA}, but it is deterministic and simpler to implement.) We call this the ``Batch'' sparse linear auto-encoder algorithm. We correspondingly implement an ``Iterative'' version with fixed sparsity $r$ in each principal component. In each step of the iterative sparse auto-encoder algorithm we use the above batch algorithm to select one principal component with sparsity at most $r$. 

We compare our 
sparse auto-encoder algorithms to the following 
state-of-the-art sparse PCA algorithms:  
\begin{description}\itemsep0pt
\item[1. TPower:] This is the truncated power method for sparse PCA described in~\cite{yuan2013truncated}. 
\item[2. Gpower-$\ell_0$:] This is the generalized power method with $\ell_0$ minimization in~\cite{journee2010generalized}. 
\item[3. Gpower-$\ell_1$:] This is the generalized power method with $\ell_1$ minimization in~\cite{journee2010generalized}.
\end{description}
All those algorithms were designed to operate for the simple $k = 1$ case (notice that our algorithms handle any $k$ without any modification); hence, 
to pick \math{k} sparse components, we use the ``deflation'' method suggested in~\cite{mackey2009deflation}: let's say $\hh_1$ is the result of some method applied to 
$\matX,$ then $\hh_2$ is the result of the same method applied to $ (\matI_n - \hh_1\hh_1\transp) \cdot \matX \cdot (\matI_n - \hh_1\hh_1\transp),$ etc. 

\subsection{Environment, implementations, datasets}
We implemented our sparse linear auto-encoder algorithms in Matlab.  
For all the other algorithms, we used the matlab implementations from~\cite{yuan2013truncated}.  
The implementations of the GPower method with $\ell_1$ or $\ell_0$ minimization is from the original paper~\cite{journee2010generalized}. 
We run all the experiments in Matlab 8.4.0.150421 (R2014b) in a Macbook machine with 2.6 GHz Intel Core i7 processor and 16 GB of RAM. 

Following existing literature, we test the above algorithms in the following three datasets (all available in~\cite{yuan2013truncated}):
1) PitProps: Here, $\matX \in \R^{n \times n}$ with $n=13$ corresponds to a correlation matrix of 180 observations measured with $13$ variables. The original dataset is described in~\cite{jeffers1967two}. 
2) Colon: This is the gene-expression dataset from~\cite{alon1999broad}; here, $\matX \in \R^{500 \times 500}$.
3) Lymphoma: This is the gene-expression dataset from~\cite{alizadeh2000distinct}; here, $\matX \in \R^{500 \times 500}$. Those are PSD matrices, hence fit the sparse PCA framework we discussed above. 


\subsection{Results}
For each dataset, we tried different  $k$ and the sparsity $r$. 
The qualitative results for different \math{k} are similar so we 
only show results for \math{k=2}.
We report the results in Figure~\ref{fig:main1}.
Notice that we plot the symmetric 
explained variance and the information loss versus
 the ``average column sparsity'' of $\matH$. This is because all algorithms 
- with TPower being an exception - 
cannot guarantee column sparsity in $\matH$ of exactly $r,$ for given $r$. 
Our methods, for example, promise column sparsity at most $r$. 
The GPower methods control the sparsity level through a 
real parameter $\gamma$ which takes values in~(0,1). 
An exact relation between $\gamma$ and $r$ is not specified, 
hence we experimented with different values of $\gamma$ 
in order to achieve different levels of sparsity. 

We show example sparse encoders \math{\matH=[\hh_1,\hh_2]} 
for the 5 algorithms with \math{k=2} and 
\math{r=5} below
{
\mand{
\setlength\arraycolsep{4pt}
\begin{array}{c|c|c|c|c}
\text{\underline{\bf Batch}}&\text{\underline{\bf Iter.}}&\text{\underline{\bf TP}}&\text{\underline{\bf GP-\math{\ell_0}}}&\text{\underline{\bf GP-\math{\ell_1}}}\\
\setlength\arraycolsep{6pt}
\renewcommand\arraystretch{1}
\begin{array}{rr}
\hh_1&\hh_2\\[2pt]
            0 &           0\\
     -0.8 &    -0.3\\
            0 &           0\\
   0 &    -0.8\\
            0 &           0\\
            0 &           0\\
            0 &           0\\
      -0.3 &     0.3\\
            0 &           0\\
            0 &           0\\
            0 &           0\\
            0 &           0\\
      0.5 &    -0.4\\
\end{array}
&
\setlength\arraycolsep{6pt}
\begin{array}{rr}
\hh_1&\hh_2\\[2pt]
           0&             0\\
     -0.6  &    -0.8\\
            0&      -0.4\\
            0&      -0.2\\
            0&             0\\
            0 &            0\\
     -0.7 &     -0.1\\
            0&             0\\
     -0.3  &           0\\
     -0.1  &          0\\
            0&             0\\
            0&      -0.2\\
            0 &            0
\end{array}
&
\setlength\arraycolsep{6pt}
\begin{array}{rr}
\hh_1&\hh_2\\[2pt]
      0.5 &           0\\
      0.5 &           0\\
            0 &     0.6\\
            0 &     0.6\\
            0 &           0\\
            0 &     0.3\\
       0.4&            0\\
            0 &           0\\
      0.4 &           0\\
      0.4 &    -0.2\\
            0 &           0\\
            0 &     0.3\\
            0 &           0
\end{array}
&
\setlength\arraycolsep{8pt}
\begin{array}{rr}
\hh_1&\hh_2\\[2pt]
      0.7   &           0\\
      0.7    &          0\\
            0  &      0.7\\
            0  &      0.7\\
            0  &            0\\
            0  &            0\\
            0 &             0\\
            0 &             0\\
            0  &            0\\
            0  &            0\\
            0  &            0\\
            0  &            0\\
            0   &           0
\end{array}
&
\setlength\arraycolsep{8pt}
\begin{array}{rr}
\hh_1&\hh_2\\[2pt]
      0.6   &          0\\
      0.6   &          0\\
            0   &    0.7\\
            0  &     0.7\\
            0  &           0\\
            0   &          0\\
            0  &           0\\
            0  &           0\\
      0.5  &           0\\
            0  &           0\\
            0  &           0\\
            0   &          0\\
            0  &           0
\end{array}
\end{array}
}%
}%
The sparse encoder vectors differ significantly among the methods.
This hints at how sensitive the optimal solution is, 
which underlines why it is important to optimize the correct
loss criterion. Since one goal of low-dimensional
feature construction is to preserve as much information as possible,
the information loss is the compeling metric.

Our algorithms give better 
information loss than existing sparse PCA approaches. 
However, existing approaches have higher symmetric
explained variance. This is in general 
true across all three datasets and 
different values of $k$. 
These findings shouldn't come at a surprise, 
since previous methods aim at optimizing symmetric explained
variance and our methods choose an encoder which optimizes
information loss. 
The figures highlight once again how different the solutions can be.
It also appears that our ``iterative'' algorithm 
gives better empirical information loss compared to the batch 
algorithm, for comparable levels of average sparsity, despite having
a worse theoretical guarantee.
%
\remove{
\paragraph{Running times.} Finally, we comment on the running times of the algorithms. Our algorithms are considerably slower than existing approaches. We do not report running times here but our algorithms are on average five to ten times slower. Notice however, that they achieve considerably better reconstruction errors, hence when the running time is not a concern and one wants as small reconstruction error as possible, our methods should be given priority over existing sparse PCA algorithms.
} 

Finally, we mention that we have not attempted to optimize the running
times (theoretically or empirically) of our algorithms.
Faster versions of our proposed algorithms
might give as accurate results as the versions we have implemented 
but with considerably faster running times; for example, in the iterative 
algorithm (which calls the CSSP algorithm with \math{k=1}), 
it should be possible
significantly speed up the generic algorithm (for arbitrary \math{k}) to 
a specialized one for \math{k=1}.
We leave such implementation optimizations for future work.

\section{Discussion}

Historically, sparse PCA has meant cardinality constrained 
variance maximization. Variance \emph{per se} does not have any intrinsic
value, and it is not easy to 
generalize to arbitrary encoders which are either not orthogonal or
not decorrelated. However, the information
loss is a natural criterion to optimize because it directly
reflects how good the features are at preserving the data. Information
loss captures the machine learning goal when 
reducing the dimension: preserve as much information as possible.

We have given efficient asymptotically optimal sparse linear encoders
An interesting open question is whether one can get a 
(\math{1+\epsilon})-relative error with respect to information loss for the
iterative encoder. We believe the answer is yes as is evidenced by
the empirical performance of the iterative encoder.

\paragraph{Acknowledgments.} We thank Dimitris Papailiopoulos for pointing
out the connection between {\sc max-clique} and sparse PCA.

{\small
\bibliographystyle{named}
\bibliography{mypapers,spca,masterbib} 
}

\end{spacing}

\end{document}